\documentclass[10pt,twocolumn,letterpaper]{article}

\usepackage{cvpr}              

\usepackage[accsupp]{axessibility}  

\usepackage{color}
\usepackage{amsmath}
\usepackage{amsthm}

\usepackage{multirow}
\usepackage{colortbl}

\usepackage{makecell}

\usepackage{titletoc}
\usepackage{tocloft}
\newcommand{\fref}[1]{Fig.~\ref{#1}}
\newcommand{\sref}[1]{Sec.~\ref{#1}}
\newcommand{\tref}[1]{Table~\ref{#1}}

\newtheorem{proposition}{Proposition}%

\newtheorem{definition}{Definition}%

\allowdisplaybreaks

\definecolor{cvprblue}{rgb}{0.21,0.49,0.74}
\usepackage[pagebackref,breaklinks,colorlinks,allcolors=cvprblue]{hyperref}

\def\paperID{} 
\def\confName{CVPR}
\def\confYear{2026}

\title{Boosting Vision-Language-Action Finetuning with\\
Feasible Action Neighborhood Prior
}

\author{
Haochen Niu$^{1,*}$ \quad
Kanyu Zhang$^{1,*}$ \quad
Shuyu Yin$^{1}$ \quad
Qinghai Guo$^{2}$ \quad
Peilin Liu$^{1}$ \quad
Fei Wen$^{1,\dagger}$\\
$^{1}$
Shanghai Jiao Tong University, China\\
$^{2}$Huawei Technologies, China\\
{\tt\small \{haochen\_niu,kane\_zhang2024,shuyu.yin,liupeilin,wenfei\}@sjtu.edu.cn}, 
{\tt\small guoqinghai@huawei.com}\\
{\small $^{*}$Equal contribution \hspace{1em} $^{\dagger}$Corresponding author}
}

\begin{document}
\maketitle
\begin{abstract}
In real-world robotic manipulation, states typically admit a neighborhood of near-equivalent actions. That is, for each state, there exist a feasible action neighborhood (FAN) rather than a single correct action, within which motions yield indistinguishable progress. However, prevalent VLA training methodologies are directly inherited from linguistic settings and do not exploit the FAN property, thus leading to poor generalization and low sample efficiency. To address this limitation, we introduce a FAN-guided regularizer that shapes the model's  output distribution to align with the geometry of FAN. Concretely, we introduce a Gaussian prior that promotes locally smooth and unimodal predictions around the preferred direction and magnitude. In extensive experiments across both reinforced finetuning (RFT) and supervised finetuning (SFT), our method achieves significant improvement in sample efficiency, and success rate in both in-distribution and out-of-distribution (OOD) scenarios. By aligning with the intrinsic action tolerance of physical manipulation, FAN-guided regularization provides a principled and practical method for sample-efficient, and generalizable VLA adaptation.
\end{abstract}
    
\addtocontents{toc}{\protect\setcounter{tocdepth}{0}}
\section{Introduction}
\label{sec:intro}

Benefiting from advances in large-scale foundation models, vision-language-action (VLA) models have recently emerged as a unifying paradigm for general robotic manipulation \cite{brohan2023RT2, belkhale2024RTH, kim2024OpenVLA, kim2025FineTuning, black2024$p_0$, intelligence2025$p_05$}. A VLA model integrates vision perception, language understanding, and low-level control within a single model, which enables an end-to-end mapping from ``seeing and understanding'' to ``acting''. A common design of VLA models leverages an autoregressive language-model backbone, by discretizing each action dimension into a set of bins and encoding the motion as ``action tokens''. 

In practice, deploying a VLA model typically involves two stages:
(i) Pretraining on large-scale heterogeneous VLA data to learn a general pretrained model. 
(ii) finetuning on a specific robot platform or environment to adapt to the system's embodiment, sensory configuration, and physical dynamics.
For the second stage, there are mainly two finetuning paradigms: supervised finetuning (SFT) and reinforced finetuning (RFT). 
SFT treats offline demonstration trajectories as labeled sequences and uses them for supervised finetuning \cite{zhao2023Learning, kim2024OpenVLA}. 
In contrast, RFT does not require demonstrations but optimizes a policy using environment rewards through interaction \cite{mark2024Policy, li2025SimpleVLARL, zang2025RLinfVLA}. 
While SFT is more stable than RFT, it has been recently shown that RFT can yield better generalization \cite{liu2025What}.

Despite the difference in supervision, most VLA training and finetuning methods directly inherit the learning recipes of language models, e.g., next-token prediction with one-hot cross-entropy or its reinforcement learning (RL) analogues such as PPO~\cite{schulman2017Proximal} and GRPO~\cite{shao2024DeepSeekMath}. However, unlike in linguistic settings, physical actions possess intrinsic \textit{tolerance} and \textit{near-equivalence}, where a neighborhood of an action can yield indistinguishable task progress.

Yet, this fundamental difference between the linguistic and robotic domains is often overlooked. This oversight causes significant challenges for both common finetuning paradigms. When using SFT with a small task-specific dataset, which is a common scenario in practice, the model learns to collapse its probability mass onto the single demonstrated action, a clear sign of severe overfitting. Meanwhile, while RFT can eventually produce a broader action distribution for better generalization, the process is highly sample-inefficient. Without an explicit regularization that exploits the intrinsic structure of the action space, the agent must spend much more training time to implicitly discover this beneficial property through exploration.

To address this limitation, we revisit VLA training from a geometric and physical-consistency perspective. 
Our approach begins by formalizing the concept of the Feasible Action Neighborhood (FAN): a set of neighboring actions for each state within which different choices yield nearly equivalent task progress. Building upon this concept, we introduce a novel regularization term designed to shape the model's output distribution to match the properties of the FAN. Specifically, we model the FAN as a locally unimodal and smooth region by encouraging the output distribution to conform to a target Gaussian. Crucially, this regularizer requires no alterations to the model architecture or auto-regressive decoding scheme, yet it effectively shifts the learning signal from a focus on ``exclusive correctness" to one of ``weighted tolerance".
This aligns the VLA training objective with the geometry of physical action spaces.

The main contributions of this work are as follows:

\begin{itemize}
\item 
We formalize the notion of FAN to characterize the local tolerance structure of physical actions, and reveal the intrinsic mismatch between standard language-style VLA training and the geometry of real-world physical actions.
\item 
We introduce a FAN-guided regularizer applicable to both SFT and RFT, which preserves the discreteness and autoregressive nature of VLA models while explicitly aligning the training objective with physical tolerance of real-world manipulation actions.
\item 
We conduct extensive evaluations across both SFT and RFT regimes, covering diverse VLA backbones and manipulation tasks under both in-distribution and out-of-distribution (OOD) settings. The results demonstrate significant improvement of the proposed method in sample-efficiency, task success rate, and OOD generalization.
\end{itemize}

Our results show that bridging the gap between language-style training objectives and physical tolerance of real-world manipulation actions benefits efficient VLA finetuning.
It is worth noting that our method is fundamentally distinct from entropy maximization prior in reinforcement learning. While entropy maximization also encourages distributional outputs, it is typically unstructured and used to promote exploration. In contrast, our method imposes a structured prior shaped by the FAN, which exploits the properties of physical actions to provide a more direct and sample-efficient path towards a generalizable policy.

\section{Related Work}

\subsection{Foundation Models for Robotics}
Early robotic systems rely on hand-crafted task specifications and fixed control policies, which have limited generalization \cite{shao2025Large}.
Recently, the progress of large foundation models \cite{awais2025Foundation} has motivated exploring lightweight adaptation of general-purpose backbones for scalable and generalizable embodied intelligence.
Initial efforts typically adopt hierarchical or decoupled paradigms. For example, VoxPoser \cite{huang2023VoxPoser} employs vision-language models (VLMs) to infer voxel-level affordances for planning, while SayCan \cite{ichter2022Can} and ViLA \cite{hu2024look} decompose natural-language goals into low-level skill executions. Moreover, Code-as-Policies \cite{liang2023Codea} uses large language models to generate robot control code.

A major paradigm shift came with RT-2 \cite{brohan2023RT2}, which unifies perception, reasoning, and control within a single VLA framework.
VLA architectures typically adopt an autoregressive language model to directly output discrete action tokens for robot execution \cite{driess2023PaLMEa, ghosh2024Octo, kim2024OpenVLA}.
Among them, OpenVLA \cite{kim2024OpenVLA} is a widely used baseline, which is pretrained at scale on the Open X-Embodiment dataset \cite{o2024open}. It fuses dual visual encoders, SigLIP \cite{zhai2023Sigmoid} and DINOv2 \cite{oquab2024dinov}, with Llama2-7B \cite{touvron2023Llama}, and autoregressively predicts 7-DoF action tokens over discretized bins.

More recent studies have refined VLA architectures to enhance reasoning, action representation, and efficiency. 
Architecturally, TriVLA \cite{liu2025TriVLA} 
leverages pretrained video prediction models for latent temporal cues, while WorldVLA \cite{cen2025WorldVLA} employs auxiliary image prediction for stronger supervision.
UniVLA \cite{bu2025UniVLA} and Villa‑X \cite{chen2025villaX} introduce additional training stages 
to reduce the influence of task‑irrelevant dynamics.
$\pi_0$ \cite{black2024$p_0$} and $\pi_{0.5}$ \cite{intelligence2025$p_05$} employ flow matching with continuous action representation, while HybridVLA \cite{liu2025HybridVLA} combines autoregressive and diffusion heads, and VQ‑VLA \cite{wang25vqvla} replaces the bin‑based tokenizer with a residual VQ‑VAE for higher precision.
To enhance reasoning, Fast‑in‑Slow \cite{chen2025FastinSlow} decouples fast manipulation from slow deliberation, while OneTwoVLA \cite{lin2025onetwovla} employs a control token to switch between reasoning and execution phases.
Efficiency‑oriented variants include SmolVLA \cite{shukor2025SmolVLA} and BitVLA \cite{wang2025BitVLA}, which employ smaller models, while EfficientVLA \cite{yang2025EfficientVLA} applies targeted pruning without retraining.
Moreover, Gr00t N1 \cite{nvidia2025GR00T} and SpatialVLA \cite{qu2025SpatialVLA} focus on sim‑to‑real transfer.

\subsection{Finetuning for VLA Models}

Despite recent growth in robot manipulation datasets, their scale remains limited compared to NLP corpora \cite{xiang2025Parallels}.
Evidences from various domains highlight the critical role of finetuning in aligning foundation models with downstream objectives, a process even more crucial for VLA models.
Most existing VLAs adopt a two-stage training paradigm comprising large-scale pretraining followed by SFT on task-specific demonstrations \cite{kim2025FineTuning}.
However, imitation-based SFT often suffers from significant degradation in OOD scenarios not covered by expert data \cite{zhou2025LIBEROPRO, fei2025LIBEROPlus}.

To overcome this limitation, recent research has drawn inspiration from reinforced finetuning successes in mathematical reasoning and code generation \cite{deepseek-ai2025DeepSeekR1, lightman2023Lets}, which applies RFT to pretrained VLA models.
Several works, such as \cite{hu2025FLaRe, liu2025What, li2025SimpleVLARL, zang2025RLinfVLA}, employ established policy optimization algorithms including PPO~\cite{schulman2017Proximal} and GRPO~\cite{shao2024DeepSeekMath}, and demonstrate their effectiveness in aligning VLA behaviors with task rewards.
ReinboT \cite{zhang2025ReinboT} considers dense reward design and reward-maximization objectives, while VLA‑RL \cite{lu2025VLARL} introduces a Robot Process Reward Model (RPRM) that replaces sparse binary rewards with pseudo-rewards derived from gripper actions and task progress, to achieve more stable PPO-based optimization.
Methods such as TGRPO \cite{chen2025TGRPO} and GRAPE \cite{zhang2025GRAPE} leverage powerful VLMs to generate rewards without training explicit reward models.
Further, Chunked RL \cite{huang2025CORFT} formulates a temporal-difference learning framework tailored to chunked action sequences, and Interactive RL \cite{tan2025Interactive} combines REINFORCE leave-one-out (RLOO) advantage estimation \cite{kool2018Attention} with PPO for more stable policy updates.
ConRFT \cite{chen2025ConRFT} combines offline Q‑learning with online human‑in‑the‑loop finetuning, while iRe‑VLA \cite{Guo2025ImprovingVM} alternates between online RL and behavior cloning in an iterative training process to balance exploration and stability.

\section{Preliminaries}

\begin{figure*}[!t]
    \centering
    \subfloat[SFT Warm-up (success rate 48.4\%)]{
        \label{subfig:radar_openvla_sft}
        \includegraphics[width=0.32\linewidth]{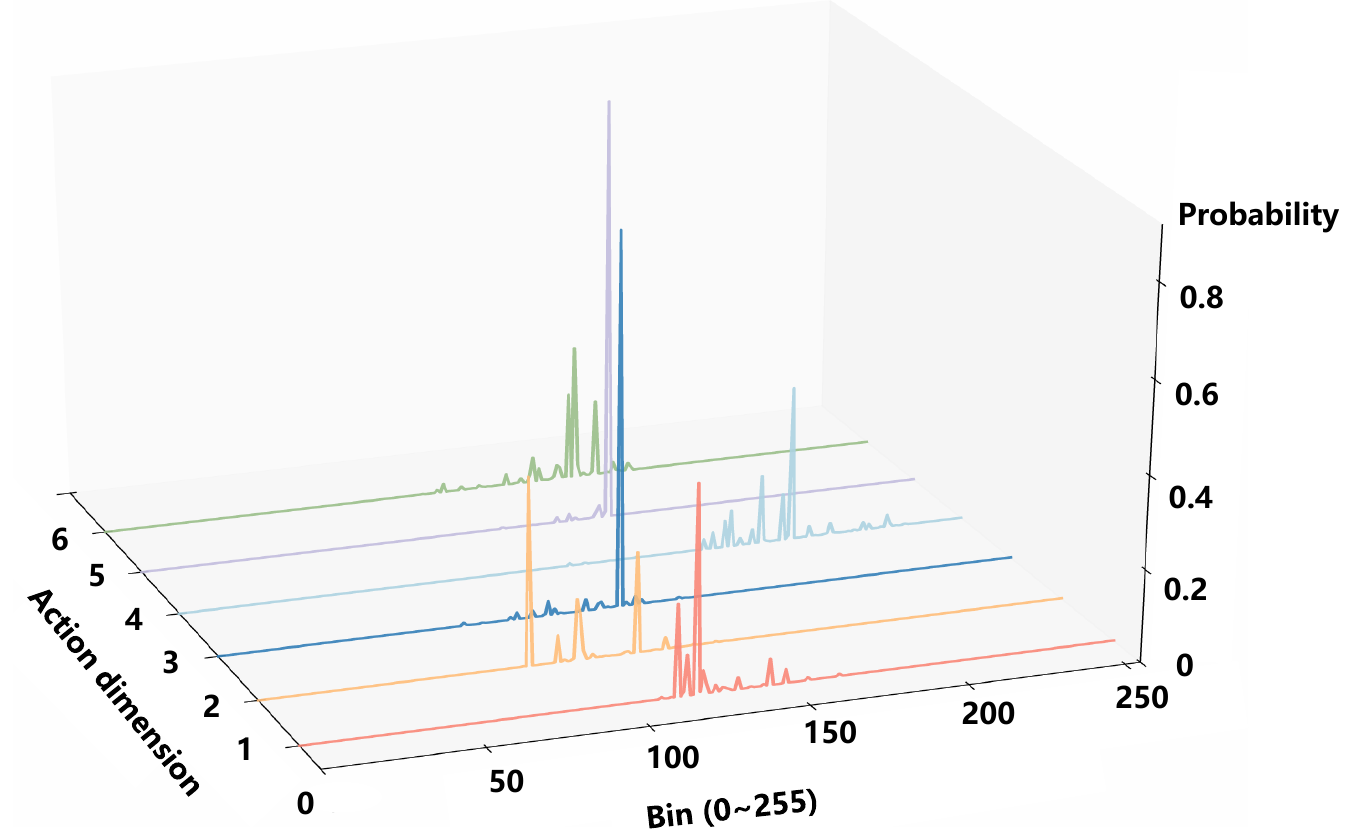}}
    \subfloat[SFT + PPO (success rate 93.8\%)]{
        \label{subfig:radar_oft_sft}
        \includegraphics[width=0.32\linewidth]{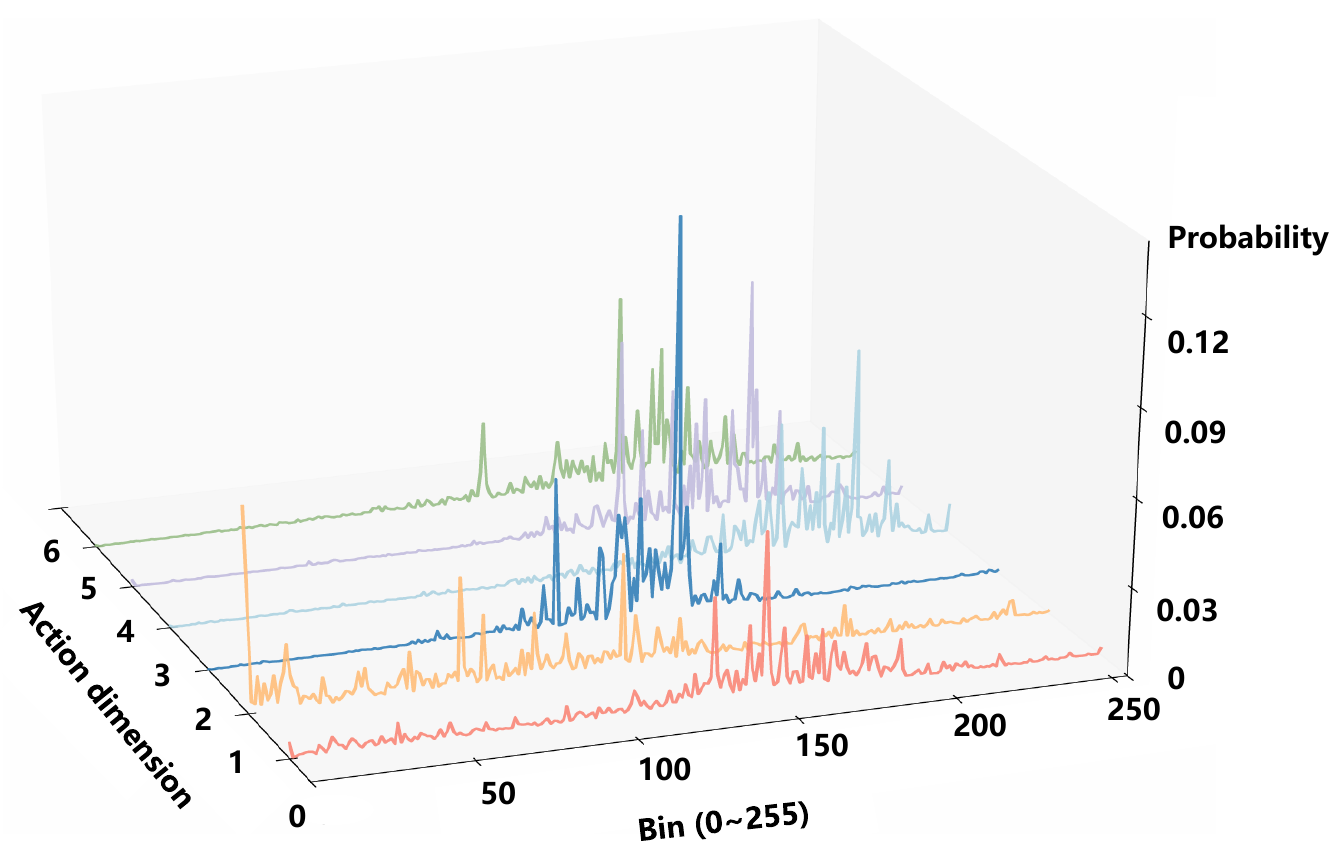}}
    \subfloat[SFT + FAN-PPO (success rate 97.4\%)]{
        \label{subfig:radar_openvla_rft}
        \includegraphics[width=0.32\linewidth]{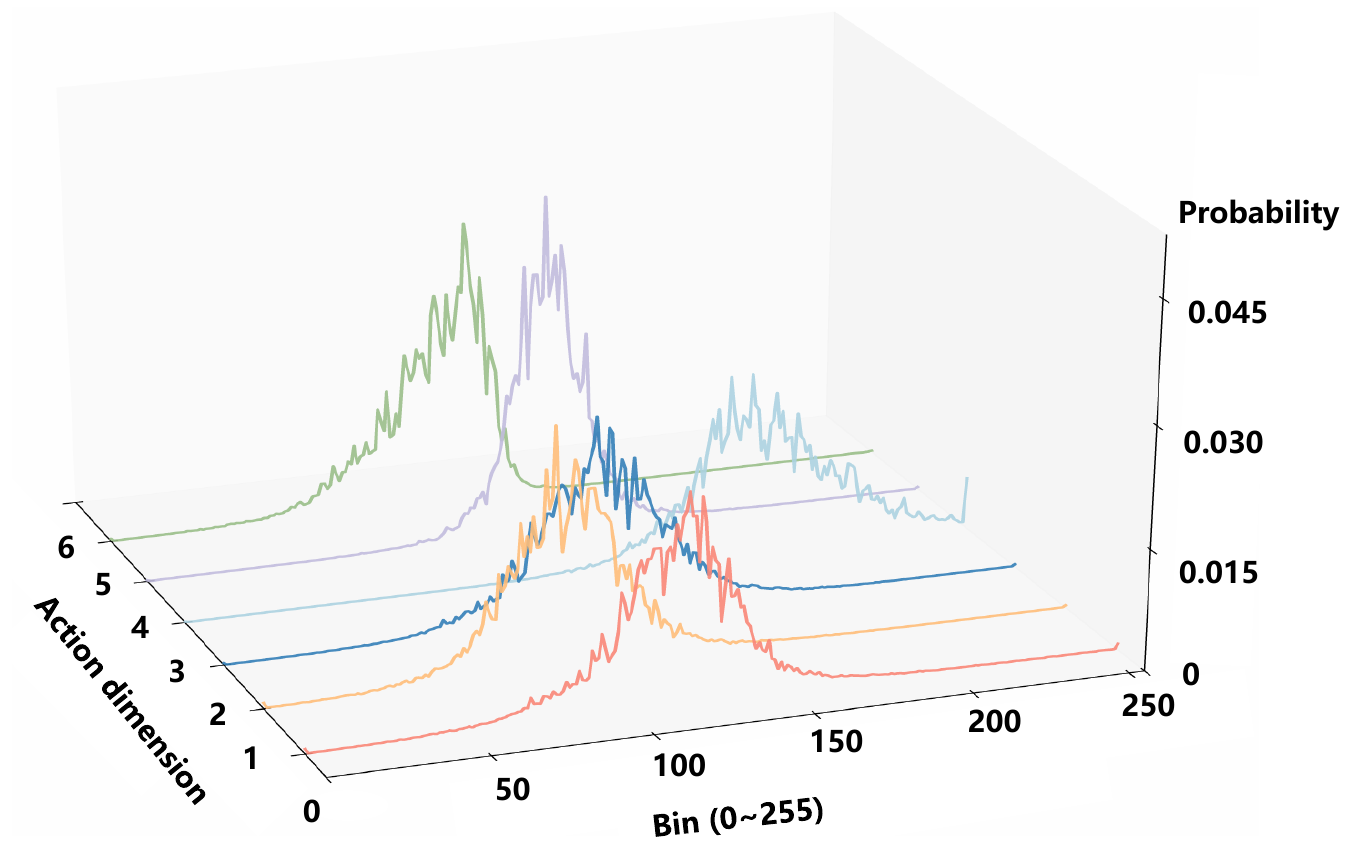}}
     \caption{
    Geometric structure of policy distribution on a ManiSkill task. (a)  After the SFT warm-up stage, the policy learned a narrow, peaked distribution with a minimal FAN, resulting in poor generalization.  (b) Subsequent RFT with PPO broadens the distribution, leading to improved task success; (c) Our FAN-PPO method explicitly guides the policy towards a robust Gaussian shape, achieving the highest success rate and demonstrating superior generalization.}
    \label{fig:fan_shape}
\end{figure*}

\paragraph{VLA as Markov Decision Processes.}
We model the VLA environment as an instruction-conditioned Markov Decision Process (MDP), defined by a tuple $(S, A, L, r, P, \gamma)$, where: $S$ represents the state space;  $A$ represents the action space; $L$ is the set of possible natural language instructions; $P: S \times A \to \Delta(S)$ is the state transition probability function, where $\Delta(S)$ is the set of probability distributions over $S$; $r: S \times A \times L \to \mathbb{R}$ is the reward function, which depends on the instruction $l \in L$; $\gamma \in [0, 1)$ is the discount factor. Without loss of generality, we assume the state space and action space both are countable sets.

The goal of this MDP is to learn an instruction-conditioned policy $\pi: S \times L \to \Delta(A)$ that maximizes the state values, conditioned on a given instruction $l \in L$. For a given policy $\pi$, instruction $l$, and initial state distribution $\mu$, the discounted state visitation distribution $d^{\pi}_{\mu}(\cdot)$ is defined as $d^{\pi}_{\mu}(s) = (1-\gamma) \sum_{t=0}^{\infty}\sum_{s',a} \gamma^t \pi(a|s',l) P_t(s|s',a) \mu(s')$, where $P_t$ is the $t$-times transition probability.

The state-action value function, $Q^{\pi}(s, a, l)$, represents the expected discounted return for taking action $a$ in state $s$ and subsequently following policy $\pi$ under instruction $l$:
$ Q^{\pi}(s,a,l) = \mathbb{E}_{\pi, P} \left[ \sum_{t=0}^{\infty} \gamma^t r(s_t,a_t) \, \middle| \, s_0 = s, a_0=a \right] $
From this, the state-value function $V^{\pi}$ and the advantage function $A^{\pi}$ are defined as the expected value of $Q^{\pi}$ over actions and the centered $Q^{\pi}$ value, respectively:
$ V^{\pi}(s, l) = \mathbb{E}_{a \sim \pi(\cdot|s,l)} [Q^{\pi}(s,a,l)] = \sum_{a \in A} \pi(a|s, l) Q^{\pi}(s,a,l) $, and  
$ A^{\pi}(s,a, l) = Q^{\pi}(s,a, l) - V^{\pi}(s, l) $.

\paragraph{Supervised finetuning (SFT).} We represent the VLA model as a policy $\pi_{\theta}$ parameterized by $\theta$. The policy is trained using SFT on a dataset of $n$ expert demonstrations, $\mathcal{D} = \{ (\tau^i, l^i)\}_{i=1}^n$. Each demonstration consists of an instruction $l^i$ and a corresponding state-action trajectory $\tau^i = \{s_0^i, a_0^i, \dots, s_{K^i-1}^i, a_{K^i-1}^i\}$, where $K^i$ is the trajectory length. The objective of SFT is to maximize the likelihood of the expert actions by minimizing the following negative log-likelihood loss:
\begin{equation}
    \label{eq:SFT_origin}
    \mathcal{L}_{\text{SFT}}(\theta) = - \frac{1}{n}\sum_{i=1}^n \sum_{t=0}^{K^i - 1} \log \pi_{\theta} (a_t^i| s_t^i, l^i).
\end{equation}

\paragraph{Trust-region-type policy optimization.} Trust-region-type policy optimization methods, such as PPO and GRPO, are widely adopted for finetuning large models, including VLAs. The core principle of these algorithms is to ensure training stability by restricting the magnitude of policy updates at each step.

This is formalized as a constrained optimization problem. At each training iteration $t$, given the current policy $\pi_t$ and the instruction $l \in L$, the objective is to find a new policy $\pi$ that maximizes the expected advantage, subject to a constraint on the policy distance (typically measured by KL divergence):
\begin{equation}
    \label{eq:trust_region_origin}
    \begin{aligned}
        \max_{\pi} \quad & \mathbb{E}_{s \sim d^{\pi_t}_{\mu}, a \sim \pi_t}\left[ \frac{\pi(a|s,l)}{\pi_t(a|s,l)} A^{\pi_t}(s,a,l) \right], \\
        \text{s.t.} \quad & \mathbb{E}_{s \sim d^{\pi_t}_{\mu}} \left[ D_{\text{KL}}(\pi(\cdot|s,l) \| \pi_t(\cdot|s,l)) \right] \leq \epsilon,
    \end{aligned}    
\end{equation}
where $D_{KL}(\cdot\| \cdot)$ represents the KL divergence. 

In practical implementation, given some trajectories sampled from the policy $\pi_{\theta_t}$, PPO employs a clip function to replace the KL divergence term. More specifically, given a sampled trajectory $\tau = \{ s_0,a_0,s_1,a_1,...,s_{K-1},a_{K-1}\}$ of length $K$, the policy loss is:
\begin{equation}
    \label{eq:PPO_origin}
    \begin{aligned}
        \mathcal{L}_{\text{PPO}}(\theta) & = - \frac{1}{K} \sum_{k=0}^{K-1} \Big[ \min \Big( I^k_t \hat{A}(s_k, a_k, l), \\
        & ~~~~~~~~~~~~ \text{Clip}(I^k_t, 1-\epsilon, 1+ \epsilon) \hat{A}(s_k, a_k, l) \Big)  \Big],
    \end{aligned}
\end{equation}
where $I^k_t = \pi_{\theta}(a_k|s_k,l)/\pi_{\theta_t}(a_k|s_k,l)$ is the importance ratio. The advantage estimation $\hat{A}(s_k,a_k)$ can be calculated using GAE~\cite{schulman2015high} or as the difference between the discounted return and a value baseline, i.e., $\hat{A}(s_k,a_k) = \sum_{i=k}^{K-1}\gamma^{i-k} r(s_i, a_i, l) - V_{\phi}(s_k, l)$. Additionally, the value function $V_{\phi}$ is updated by minimizing the mean squared error loss:
\begin{equation}
    \label{eq:PPO_value_loss}
    \mathcal{L}_{\text{Val}}(\phi) = \frac{1}{K} \sum_{k=0}^{K-1} \left( V_{\phi}(s_k, l) - \sum_{i=k}^{K-1} \gamma^{i-k} r(s_i,a_i,l) \right)^2.
\end{equation}

\section{Feasible Action Neighborhood}
\label{sec:FAN}

A defining characteristic of real-world robotic manipulation, which distinguishes it from purely linguistic tasks, is the inherent tolerance in physical actions. 
For any given state, there typically exists a neighborhood of actions around the optimal action, not just a single action, that can lead to near-identical and successful outcomes. For instance, when grasping an object, if the recorded optimal action is to move the gripper $1.0 \, cm$ to the left, moving it $0.9 \, cm$ or $1.1 \, cm$ would likely be just as effective. To formalize this tolerance region, we utilize the state-action value function, $Q(s, a)$, which represents the expected long-term performance:
\begin{definition}[Feasible Action Neighborhood, FAN]
    \label{def:fan}
    Let $a^*(s) = \arg\max_{a' \in A} Q(s,a')$ be the optimal action for a given state $s$. For a tolerance $\delta > 0$, the FAN, denoted by $\mathbb{N}_{\delta}(s)$, is the connected set containing $a^*(s)$ within the set of actions that yield nearly equivalent $Q$-values:
    \begin{equation*}
        \mathbb{N}_{\delta}(s) \subseteq \left\{a \in A : Q(s,a^*(s)) - Q(s,a) \le \delta \right\}.
    \end{equation*}
    We posit that for any well-posed physical task, this neighborhood is non-trivial, i.e., it forms a region of non-zero volume around $a^*(s)$.
\end{definition}
While the FAN is formally defined via the $Q$-function, the policy $\pi(a|s)$ implicitly encodes this information. For common policy formulations (e.g., a softmax over $Q$-values), actions with similar $Q$-values are assigned similar probabilities. This implies that the shape of the policy distribution reflects the structure of the FAN. Consequently, the shape of the policy distribution $\pi(a|s)$ serves as a practical and observable proxy for the FAN that the model has learned. A ``spiky'' narrow distribution corresponds to a small, trivial FAN, while a ``broad'' smooth distribution indicates a larger, more robust FAN.

Our empirical analyses reveal a striking correlation between the shape of the policy distribution, which is also the size of its implicit FAN, and the model's generalization performance, as shown in Figure \ref{fig:fan_shape}. In (a), we observe that the model finetuned with SFT on a small dataset for warming up consistently produces overly-peaked distributions. This indicates it has learned a trivial or near-trivial FAN, leading to poor generalization under perturbation. In contrast, in (b), when the same model is subsequently trained using RFT, its action distributions become broader, indicating the emergence of a larger, more robust FAN. Critically, this broadening of the distribution is directly correlated with a significant improvement in task success rate. This strong empirical evidence, that the policy's shape is a proxy for robustness, forms the primary motivation for our method.

\section{Methodology}

As established in Section \ref{sec:FAN}, shaping the policy distribution, which serves as a proxy for the FAN, is crucial for improving generalization and {sample efficiency} during finetuning. Our goal is therefore to guide the policy distribution during training to reflect the properties of an ideal FAN.

The key properties of a physical FAN are unimodality, smoothness, and local contiguity. These properties are well-modeled by a Gaussian distribution. Therefore, we introduce a FAN-guided regularizer, defined as the KL divergence between the policy $\pi$ and a target Gaussian distribution $\mathcal{N}(\mu(s),\Sigma(s))$:
\begin{equation}
    \label{eq:self_gaussianing_reg}
    \mathcal{L}_{\text{FAN}} = \mathbb{E}_{s}\left[D_{\text{KL}}(\pi(\cdot|s) \,\|\, \mathcal{N}(\cdot|\mu(s), \Sigma(s))) \right],
\end{equation}
where $\mathcal{N}(\cdot|\mu(s), \Sigma(s))$ denotes the density function of the normal distribution, the mean $\mu(s) = \arg\max_{a} \pi(a|s)$ is the policy's own predicted optimal action for state $s$, and the covariance matrix $\Sigma(s)$ directly controls the size of the FAN. This regularizer can be seamlessly incorporated into both SFT and RFT objectives to actively shape the policy distribution around its peak, transforming it from an overly confident ``spike" into a smooth, robust neighborhood.

\subsection{SFT with FAN-Guided Regularization}

We first incorporate our FAN-guided regularizer into the SFT loss \eqref{eq:SFT_origin}. This results in the FAN-SFT loss:
\begin{align*}
    \mathcal{L}_{\text{FAN-SFT}}(\theta) & = - \frac{1}{n}\sum_{i=1}^n \sum_{t=0}^{K^i - 1} \Big( \log \pi_{\theta} (a_t^i| s_t^i, l^i) \\
    &  + \alpha D_{\text{KL}}(\pi_{\theta} (\cdot| s_t^i, l^i) \| \mathcal{N}(\cdot|\mu(s_t^i),\Sigma(s_t^i)) ) \Big),
\end{align*}
where $\alpha$ is the regularization coefficient. Here, the covariance is not fixed but is dynamically defined by the variance of the policy: $\Sigma(s)  = \text{diag}\left( \sum_{a \in A} \pi(a|s, l) ( a - \mu(s) )^2 \right)$. This adaptive covariance matrix encourages the policy to adopt a Gaussian shape according to its current geometric property. We employ this dynamic target specifically during SFT because the supervised objective is inherently stable and can accommodate such a variable target without compromising training stability.

\subsection{RFT with FAN-Guided Regularization}

We now adapt our approach for the RFT paradigm. The FAN-guided regularizer is incorporated into a trust-region-type policy optimization objective, creating a principled trade-off between policy improvement and policy shaping.

At each update step $t$, we optimize the policy $\pi$ based on experience gathered with the current policy $\pi_t$. Our regularized objective is formulated as the following constrained optimization problem:
\begin{equation}
    \label{eq:fine_turning_with_RL}
    \begin{aligned}
        \max_{\pi} \quad & \mathbb{E}_{s \sim d^{\pi_t}_{\mu}, a \sim \pi_t}\left[ \frac{\pi(a|s, l)}{\pi_t(a|s, l)} A^{\pi_t}(s,a, l) \right] \\
        & ~~ - \alpha \mathbb{E}_{s \sim d^{\pi_t}_{\mu}} \left[  D_{\text{KL}}(\pi(\cdot|s,l) \| \mathcal{N}(\cdot|\mu(s),\Sigma)) \right], \\
        ~ \text{s.t.} \quad & \mathbb{E}_{s \sim d^{\pi_t}_{\mu}} \left[ D_{\text{KL}}(\pi(\cdot|s, l) \| \pi_t(\cdot|s, l))\right] \leq \epsilon ,
    \end{aligned}
\end{equation}
where $\alpha$ is the regularization coefficient. Crucially, to ensure training stability, we use a fixed covariance matrix $\Sigma = \sigma^2 I$. Here, $\sigma > 0$ is a fixed hyperparameter controlling the target FAN size, and $I$ is the identity matrix. This contrasts with our SFT approach, providing a stable, consistent target shape that helps anchor the policy updates.

The solution to the constrained optimization problem in \eqref{eq:fine_turning_with_RL} has an elegant closed-form structure, which reveals the mechanism of our regularizer.
\begin{proposition}[Form of the optimal policy]
    \label{prop:form_of_optimal_policy}
    For a given state $s$, and any action $a \in A$, the optimal policy $\pi_{t+1}$ that solves objective \eqref{eq:fine_turning_with_RL} is given by:
    \begin{equation*}
        \begin{aligned}
            \pi_{t+1}(a|s,l) & \propto  \mathcal{N}(a|\mu(s),\Sigma)^{\frac{\alpha}{\alpha + \beta^*}} \pi_t(a|s,l)^{\frac{\beta^*}{\alpha + \beta^*}}  \\ 
            & ~~~~~~ \times \exp\left( \frac{Q^{\pi_t}(s,a,l)}{\alpha + \beta^*}\right),
        \end{aligned} 
    \end{equation*}
    where $\beta^* \geq 0$ is the optimal Lagrange multiplier for the trust-region constraint.
\end{proposition}

The proof is in the supplementary material. Proposition \ref{prop:form_of_optimal_policy} provides a key insight: the next policy, $\pi_{t+1}$, is a geometric interpolation between the previous policy $\pi_t$ and the target Gaussian $\mathcal{N}(\mu(s), \Sigma)$, where the result is then re-weighted by the exponentiated $Q$-values. The interpolation weights are determined by $\alpha$ and $\beta^*$. As the optimal multiplier, $\beta^*$ is governed by the KKT conditions, which dictate that $\beta^*$ is inversely related to the trust-region size $\epsilon$. A smaller $\epsilon$ forces a larger $\beta^*$, which makes the update more conservative by placing more weight on the previous policy $\pi_t$. Consequently, there is a competition: $\alpha$ governs the strength of the pull towards the Gaussian shape, while $\epsilon$ governs the strength of the pull towards the $\pi_t$. Eventually, finetuning with the FAN-PPO loss guides the policy distribution to approximate the target Gaussian, as illustrated in Figure \ref{fig:fan_shape} (c).

For a practical implementation, we integrate our regularizer into the policy loss  \eqref{eq:PPO_origin} to define the FAN-PPO loss:
\begin{equation*}
    \begin{aligned}
        \mathcal{L}_{\text{FAN-PPO}}(\theta) & = - \frac{1}{K} \sum_{k=0}^{K-1} \Big[ \min \Big( I^k_t \hat{A}(s_k, a_k, l), \\
        & ~~~~~~~~~~~~ \text{Clip}(I^k_t, 1-\epsilon, 1+ \epsilon) \hat{A}(s_k, a_k, l) \Big)  \\
        & ~~~~~~~  - \alpha D_{\text{KL}}(\pi_{\theta}(\cdot|s_k,l) \| \mathcal{N}(\cdot|\mu(s_k),\Sigma)) \Big],
    \end{aligned}
\end{equation*}
where $I^k_t = \pi_{\theta}(a_k|s_k,l)/\pi_{\theta_t}(a_k|s_k,l)$ is the probability ratio. The advantage $\hat{A}$ is estimated via GAE and the value function is trained using the loss from \eqref{eq:PPO_value_loss}. 

\section{Experiments}
\begin{table*}[!t]
    \centering
    \small
    \caption{Comparison of SFT results on the ManiSkill benchmark. Values denote success rates (\%).} 
    \label{tab:ms_sft}
    \renewcommand{\arraystretch}{0.85}
    \begin{tabular}{ll c cccc}
    \toprule
    & \multirow{2}{*}{\textbf{Method}}
    & \multirow{2}{*}{\textbf{In-Distribution}}
    & \multicolumn{4}{c}{\textbf{Out-of-Distribution}} \\
    \cmidrule(r){4-7}
    &
    &
    & \textbf{Vision} & \textbf{Semantic} & \textbf{Execution} & \textbf{Avg.} \\
    \midrule
    & RL4VLA \cite{liu2025What} & 88.5 & 74.0 & 61.8 & 46.2 & 60.7 \\
    \midrule
    & OpenVLA + SFT \cite{kim2024OpenVLA} & 78.1 $\pm$ 3.1  & 76.6 $\pm$ 1.9 & 57.4 $\pm$ 0.9 & 40.4 $\pm$ 0.8 & 58.1 \\
    & OpenVLA + FAN-SFT (Ours) & 89.8 $\pm$ 0.8 & 81.7 $\pm$ 1.1   & 63.5 $\pm$ 1.5   & 44.8 $\pm$ 0.5 & 63.3 \\
    \rowcolor{lightgray}
    & $\triangle$ Improvement & +11.7 & +5.1 & +6.1 & +4.4 & +5.2 \\
    \bottomrule
    \end{tabular}
\end{table*}

\begin{figure*}[!t]
    \centering
    \includegraphics[width=0.95\textwidth]{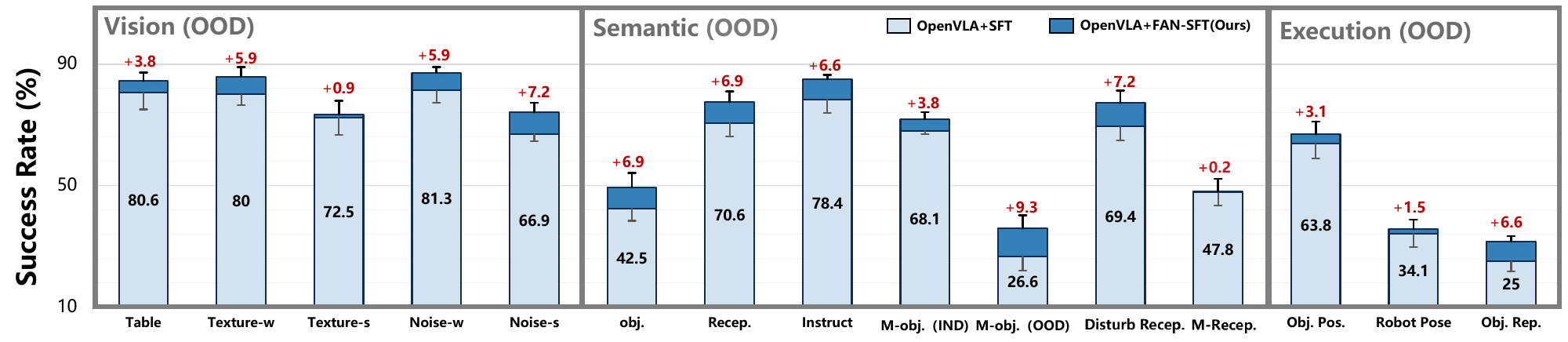}
    \caption{SFT performance on OpenVLA with and without FAN-guided regularization across different OOD tasks on ManiSkill.
    }
    \label{fig:ms_ood}
\end{figure*}

\begin{figure*}[!t]
    \centering
    \subfloat[In-Distribution]{
        \label{subfig:radar_openvla_sft}
        \includegraphics[width=0.24\linewidth]{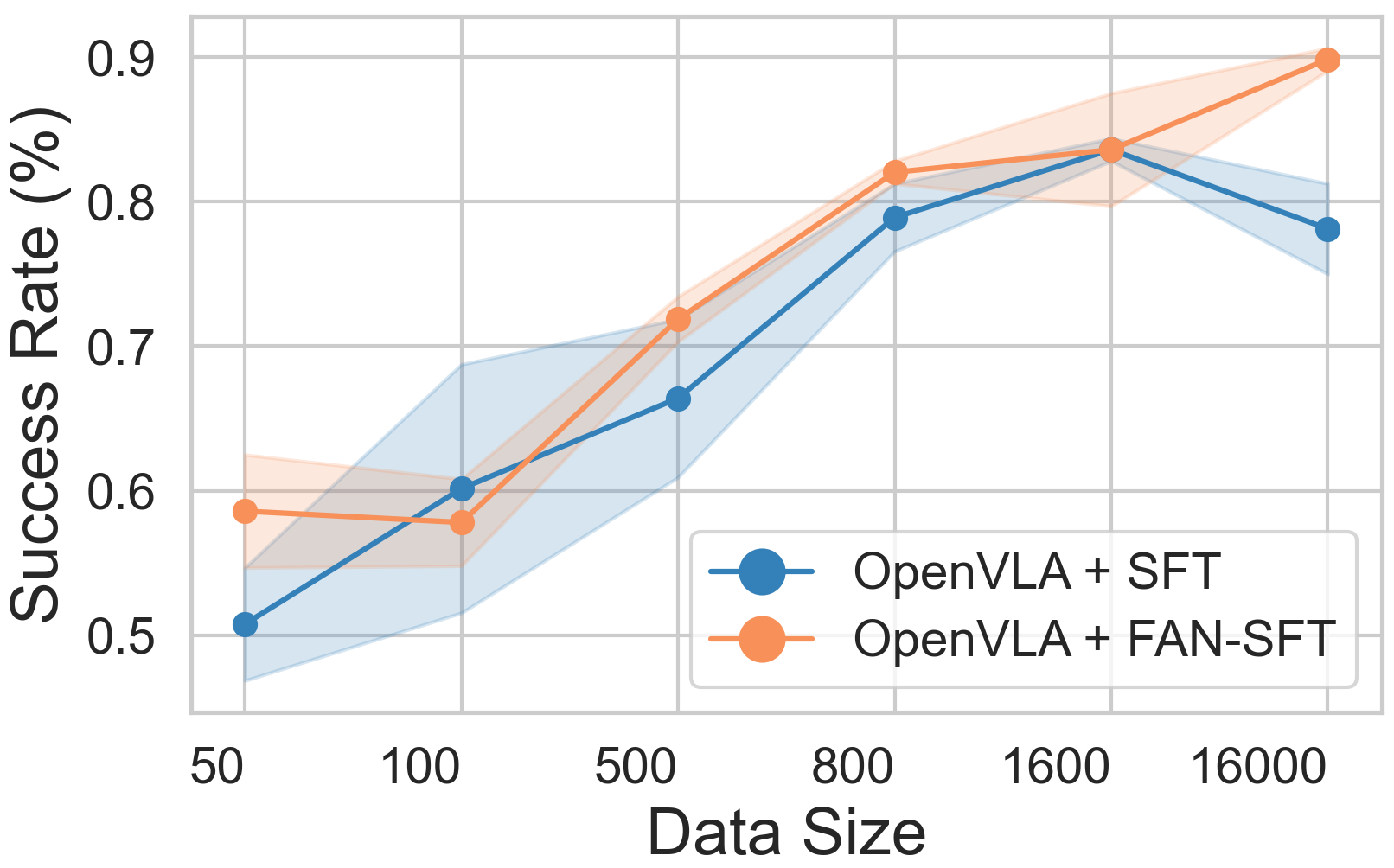}}
    \subfloat[Vision (OOD)]{
        \label{subfig:radar_oft_sft}
        \includegraphics[width=0.24\linewidth]{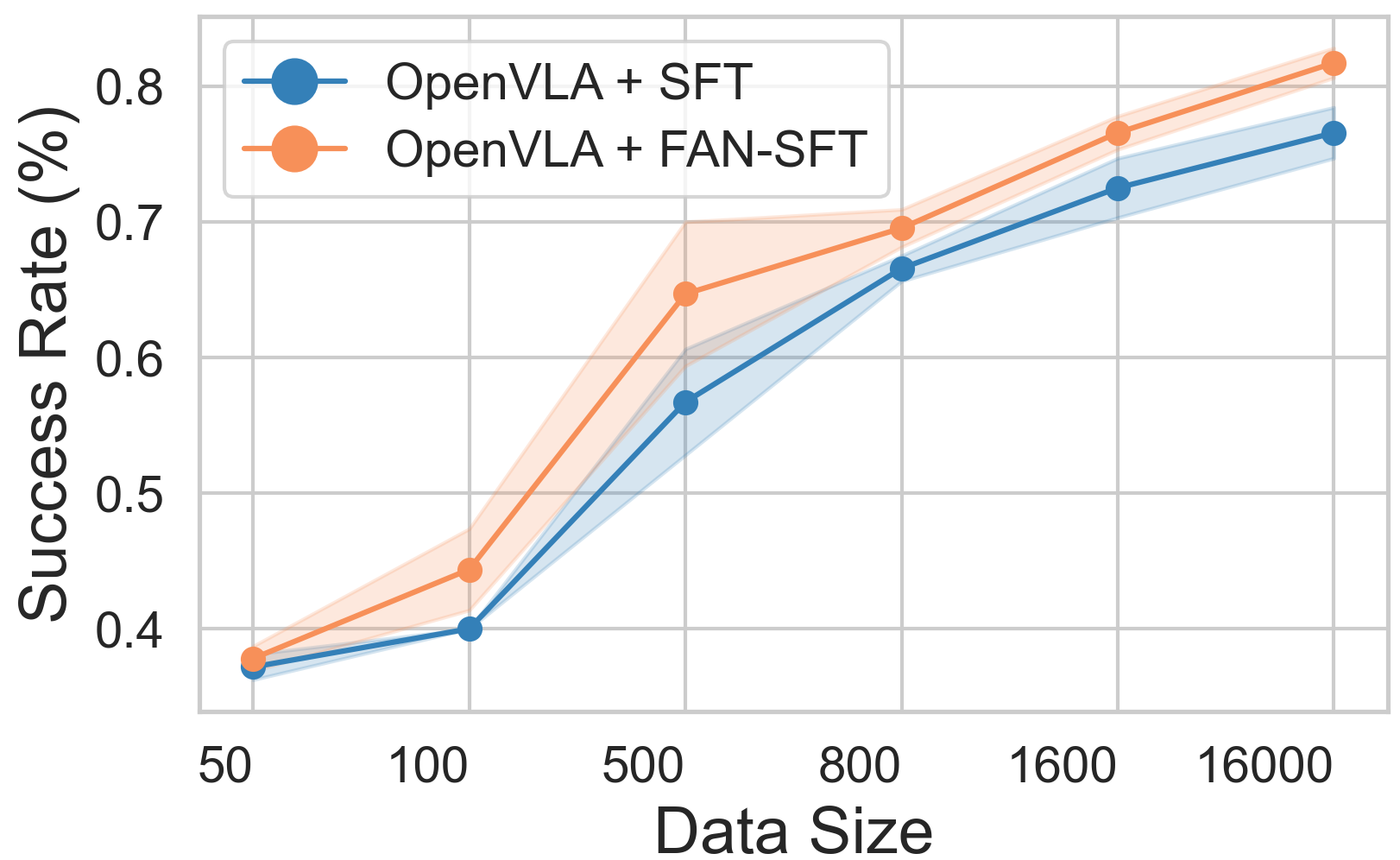}}
    \subfloat[Semantic (OOD)]{
        \label{subfig:radar_openvla_rft}
        \includegraphics[width=0.24\linewidth]{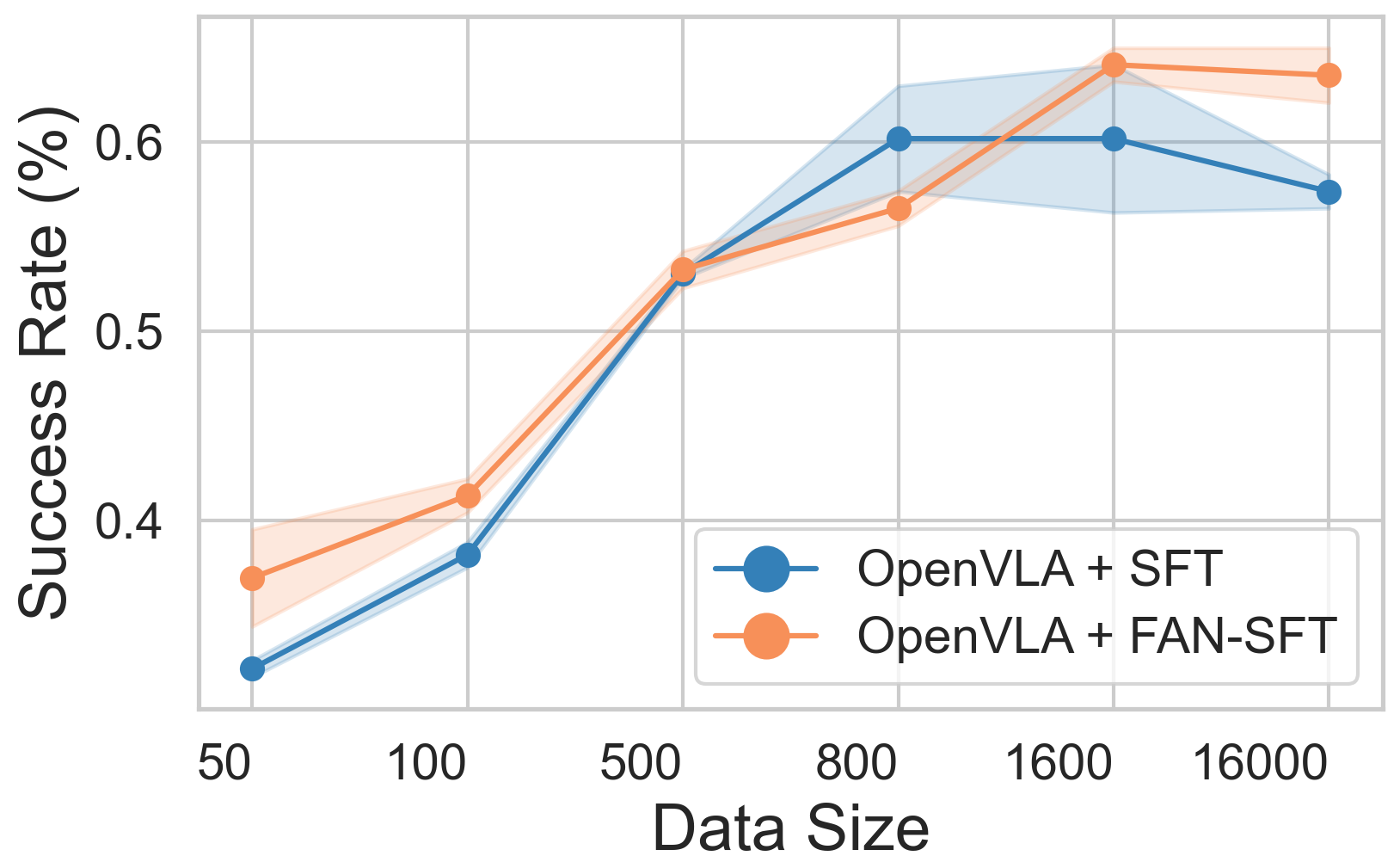}}
    \subfloat[Execution (OOD)]{
        \label{subfig:radar_oft_rft}
        \includegraphics[width=0.24\linewidth]{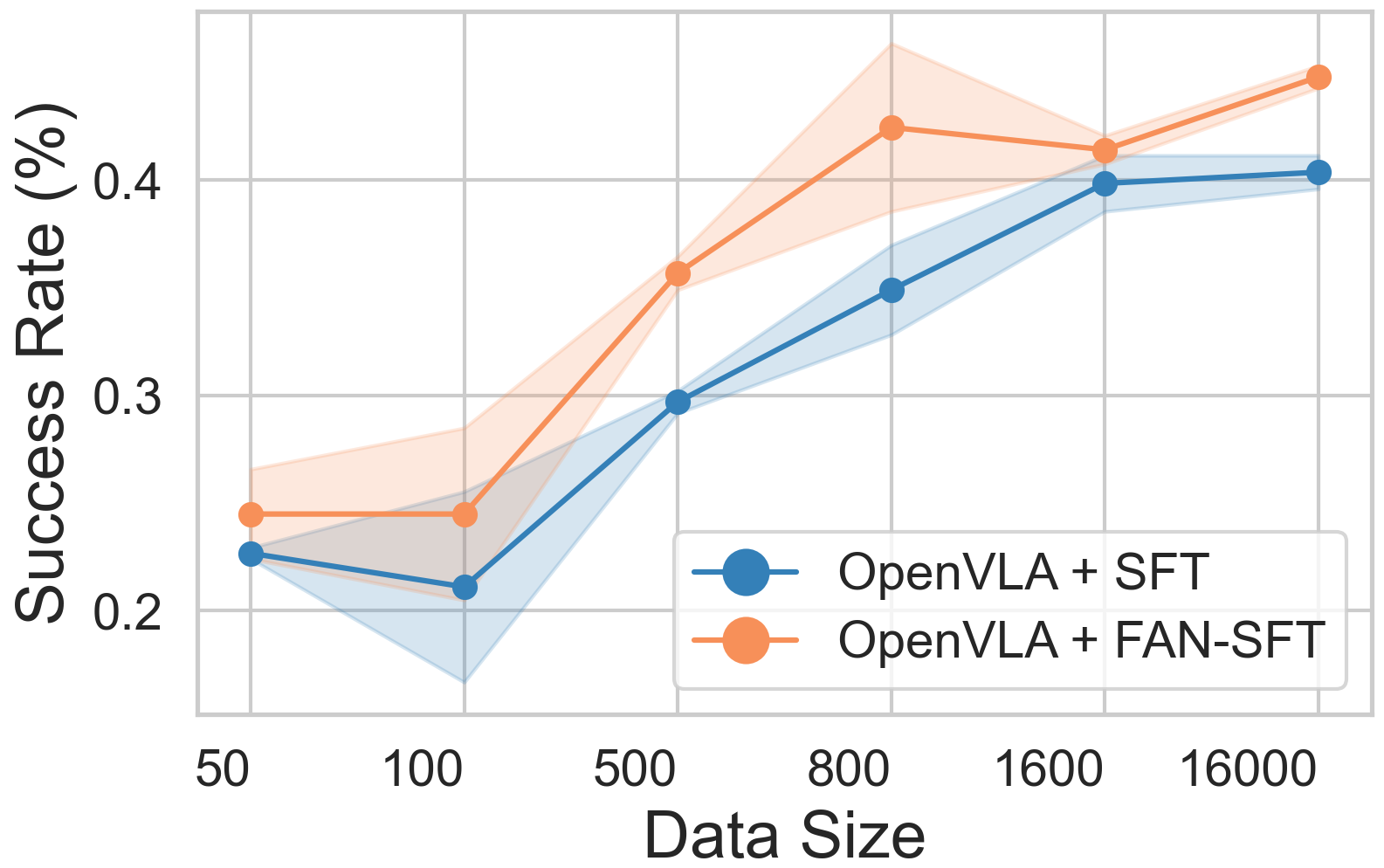}}
    \caption{SFT performance on OpenVLA  with and without FAN-guided regularization across different data sizes on in-distribution and three OOD tasks (Vision, Semantic, Execution).
    }
    \label{fig:scale}
\end{figure*}

We conduct experiments on both SFT  and RFT with two representative VLA models: OpenVLA \cite{kim2024OpenVLA}, which outputs a single action, and OpenVLA-OFT \cite{kim2025FineTuning}, which outputs action chunks. We consider two widely used robot manipulation benchmarks: Maniskill \cite{tao2025ManiSkill3} and LIBERO \cite{liu2023LIBERO}. 
{We comprehensively evaluate the proposed FAN-guided regularization in terms of absolute performance gains, OOD generalization, sample efficiency, and convergence speed.}
All experiments are conducted on NVIDIA A100 GPUs with 80GB memory. 

\subsection{Supervised Finetuning}\label{subsec:sft}

\begin{figure}[t]
    \centering
    \includegraphics[width=0.48\textwidth]{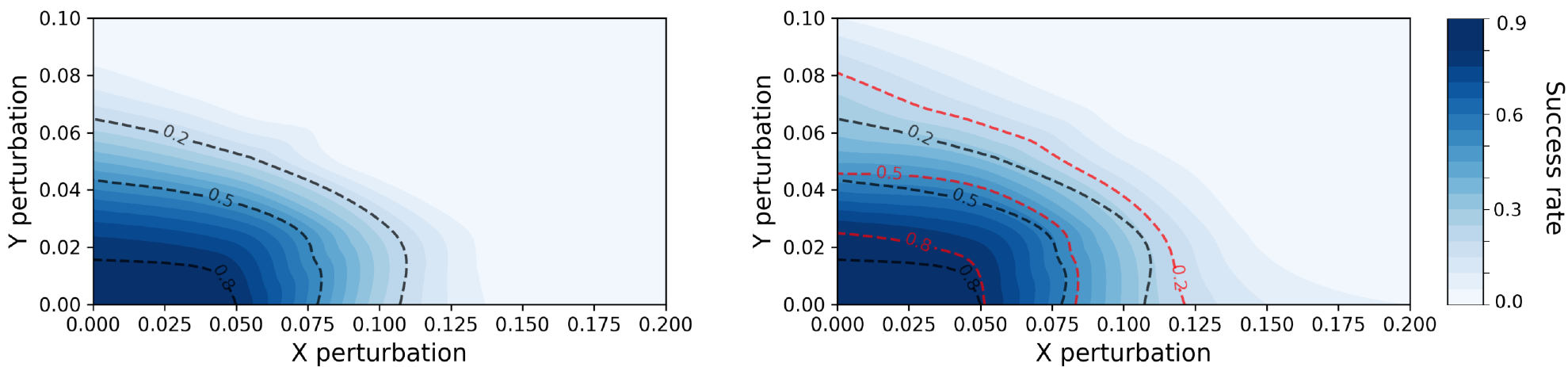}
    \caption{Spatial robustness on \texttt{LIBERO-Spatial}, comparing OpenVLA finetuned with SFT (left) versus our FAN-SFT (right). Color indicates success rate; the black and red dashed lines are the equal-success-rate contours for each method, respectively.}
    \label{fig:libero_ood}
\end{figure}

\begin{figure}[!t]
    \centering
    \subfloat[SFT]{
        \label{subfig:radar_openvla_sft}
        \includegraphics[width=0.3\linewidth]{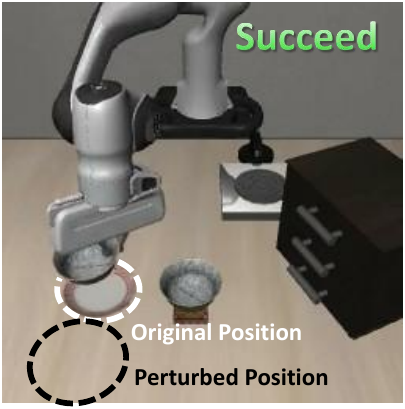}}
    \subfloat[SFT (OOD)]{
        \label{subfig:radar_oft_sft}
        \includegraphics[width=0.3\linewidth]{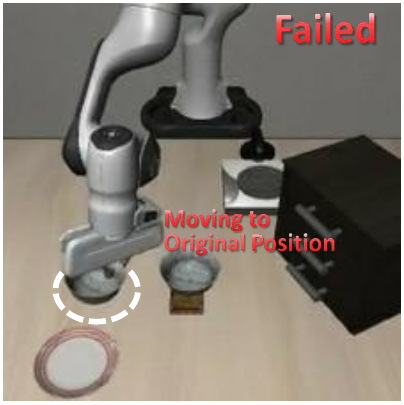}}
    \subfloat[FAN-SFT (OOD)]{
        \label{subfig:radar_openvla_rft}
        \includegraphics[width=0.3\linewidth]{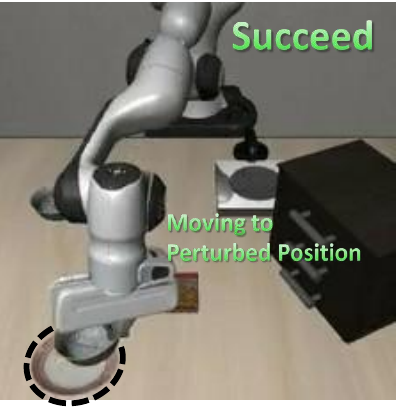}}
        \caption{{Qualitative comparison under spatial perturbation.
        Vanilla SFT remains biased toward the seen position, whereas FAN-SFT correctly adapts to the perturbed target location.}
        }
    \label{fig:libero_example}
\end{figure}

\paragraph{Implementation details.}
On ManiSkill, following \cite{liu2025What}, we train on \texttt{PutOnPlateInScene25Main-v3}, which comprises 25 diverse pick-and-place manipulation primitives across multiple object categories and target receptacles.
We collected a training dataset of the same scale (16K) as RL4VLA \cite{liu2025What} using a motion planner within the environment.
We use a pre-trained OpenVLA checkpoint from HuggingFace\footnote{ \url{https://huggingface.co/gen-robot/openvla-7b-rlvla-warmup}}.
For robustness evaluation, we assess performance under 15 perturbation types spanning visual, semantics, and execution variations.

For the LIBERO benchmark, we finetune models on the \texttt{LIBERO-Spatial} suite, which consists of 10 tasks, each provided with 50 human-teleoperated demonstrations. 
We use pre-trained OpenVLA and OpenVLA-OFT checkpoints from HuggingFace\footnote{\url{https://huggingface.co/openvla/openvla-7b}}.
Among them, OpenVLA-OFT uses the variant that simultaneously takes as input the third-person image, wrist camera image, robot proprioceptive state, and language instruction.
Additionally, we follow the setting of LIBERO-PRO \cite{zhou2025LIBEROPRO} and introduce varying degrees of spatial perturbations into the task setup to evaluate the model’s generalization under position-shifted OOD conditions.

\paragraph{Main Results.}
\begin{table*}[!t]
    \centering
    \small
    \caption{Comparison of RFT results on the ManiSkill benchmark. Values denote success rates (\%).}
    \label{tab:ms_rl}
    \renewcommand{\arraystretch}{0.85}
    \begin{tabular}{ll c cccc}
    \toprule
    & \multirow{2}{*}{\textbf{Method}}
    & \multirow{2}{*}{\textbf{In-Distribution}}
    & \multicolumn{4}{c}{\textbf{Out-of-Distribution}} \\
    \cmidrule(r){4-7}
    &
    &
    & \textbf{Vision} & \textbf{Semantic} & \textbf{Execution} & \textbf{Avg.} \\
    \midrule
    & RL4VLA\cite{liu2025What} & 93.8 & 76.6 & 72.8 & 78.3 & 75.9 \\
    \midrule  
    & OpenVLA + PPO & 95.9 $\pm$ 3.2 & 80.1 $\pm$ 0.1 & 79.7 $\pm$ 2.0 & 85.8 $\pm$ 1.8 & 81.9 \\
    & OpenVLA + FAN-PPO (Ours) & 97.4 $\pm$ 0.7 & 85.0 $\pm$ 4.0 & 86.7 $\pm$ 1.3 & 92.6 $\pm$ 1.5 & 88.1 \\
    \rowcolor{lightgray}
    & $\triangle$ Improvement & +1.5 & +4.9 & +7.0 & +6.9 & +6.2 \\
    \midrule
    & OpenVLA-OFT + PPO & 92.3  $\pm$ 2.5 &  84.9 $\pm$ 1.1  & 49.0 $\pm$ 0.6 & 55.9 $\pm$ 1.2 & 63.3 \\
    & OpenVLA-OFT + FAN-PPO (Ours) & 97.3 $\pm$ 1.3 & 88.1 $\pm$ 2.2 & 58.6 $\pm$ 1.0 & 67.0 $\pm$ 2.2 & 71.2 \\
    \rowcolor{lightgray}
    & $\triangle$ Improvement & +5.0 & +3.2 & +9.6 & +11.1 & +7.9 \\
    \bottomrule
    \end{tabular}
\end{table*}

\begin{figure*}[t]
    \centering
    \includegraphics[width=0.93\textwidth]{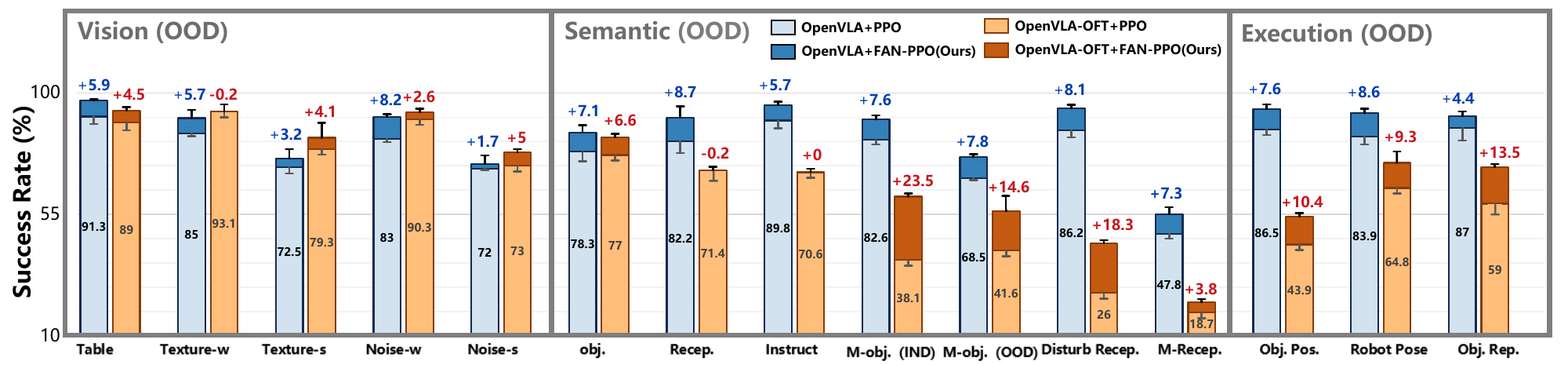}
    \caption{Performance comparison of OpenVLA and OpenVLA-OFT trained with PPO and FAN-PPO across various OOD tasks.}
    \label{fig:rl_ood}
\end{figure*}

\tref{tab:ms_sft} shows that, with FAN-guided regularization, FAN-SFT 
achieves average improvements of +11.7\% and +5.2\% on in-distribution and OOD tasks, respectively.
As detailed in \fref{fig:ms_ood}, our method achieves consistent improvements across all OOD tasks, with the largest absolute gain of 7.2\% observed on the Disturb Recep (distractive receptacle) task. These results demonstrate that the proposed regularization significantly enhances the model’s robustness under diverse perturbations.
%
{We also compare against label smoothing\cite{szegedy2016rethinking}, and the results show that FAN, as a structured regularizer, is more effective. Detailed results are provided in Appendix \tref{tab_sup:ls}.}

We further investigate the effectiveness of FAN-guided regularization under different training data scales. From the collected SFT data, we construct randomly sampled subsets at multiple data scales. 
{Similar to \cite{liu2025What}, we train for 60K steps when the dataset size is 1.6K or 16K, and for 20K steps for smaller subsets due to data scarcity.}
\fref{fig:scale} presents the average results on in-distribution and three types of OOD tasks. While the performance of both methods generally improves with increasing data scale, certain fluctuations are observed. Across most data scales, our FAN-guided regularization outperforms the baseline, which demonstrates its effectiveness under varying data sizes.

On LIBERO, we apply our method to both OpenVLA and OpenVLA-OFT models. 
On the \texttt{LIBERO-Spatial} suite, our FAN-guided regularization improves the success rate from 84.7\% to 87.2\% on OpenVLA and from 95.2\% to 98.8\% on OpenVLA-OFT. 
We further introduce varying degrees of perturbations along the \textit{x} and \textit{y} axes to evaluate the robustness of OpenVLA finetuned with SFT and {FAN-SFT.
{As shown in \fref{fig:libero_ood}, FAN-guided regularization yields stronger resilience to positional shifts, increasing the success rate from 0.24 to 0.36 at a perturbation of (\textit{x}=0.05, \textit{y}=0.05), for example.}
The qualitative results in \fref{fig:libero_example} further illustrate this effect: 
after spatial perturbation, the model trained with standard SFT tends to place the bowl at the same plate position seen during training, which indicates overfitting to familiar spatial patterns. 
In contrast, with FAN-SFT, the model can successfully adapt to unseen positions, demonstrating stronger spatial generalization ability.
%
{The complete LIBERO results are provided in Appendix \tref{tab_sup:libero_sft_ind}}.

\subsection{Reinforced Finetuning}

\paragraph{Implementation details.}

Considering that ManiSkill supports GPU-parallelized simulation, we conduct RFT on this benchmark. 
Similar to \sref{subsec:sft}, we interact within the \texttt{PutOnPlateInScene25Main-v3} environment and subsequently evaluate OOD performance under 15 types of perturbations.

In the proposed FAN-PPO method, two hyperparameters play crucial roles in controlling the stochasticity and adaptation strength. 
For OpenVLA, all results are obtained with $\sigma = 0.3$ and $\alpha = 1.0$. 
For OpenVLA-OFT, we use $\sigma = 0.2$ and $\alpha = 0.1$. 
We use the same pre-trained OpenVLA checkpoints as in previous experiments, and the OpenVLA-OFT weights are publicly available\footnote{\url{https://huggingface.co/RLinf/RLinf-OpenVLAOFT-ManiSkill-Base-Lora}
}. 

\paragraph{Main Results.}

\begin{figure}[!t]
    \centering
    \subfloat[Rollout on the training set.\label{fig:ms_oft_rl}]{
        \includegraphics[width=0.47\linewidth]{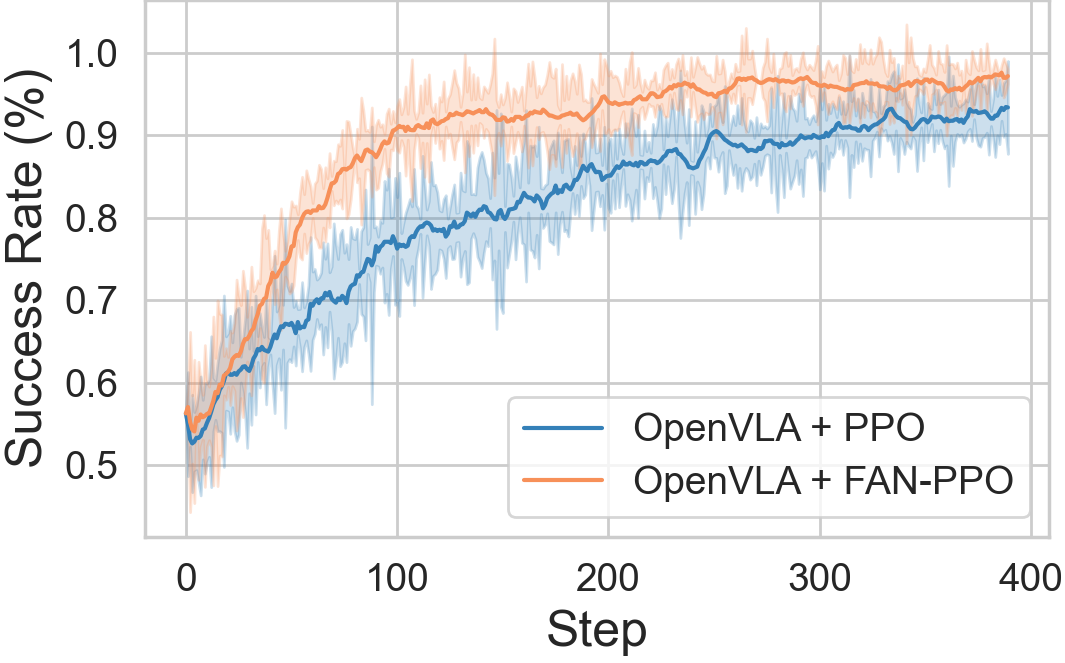}}
    \hfill
    \subfloat[Evaluation on the test set.\label{fig:ms_oft_rl_ood}]{
        \includegraphics[width=0.47\linewidth]{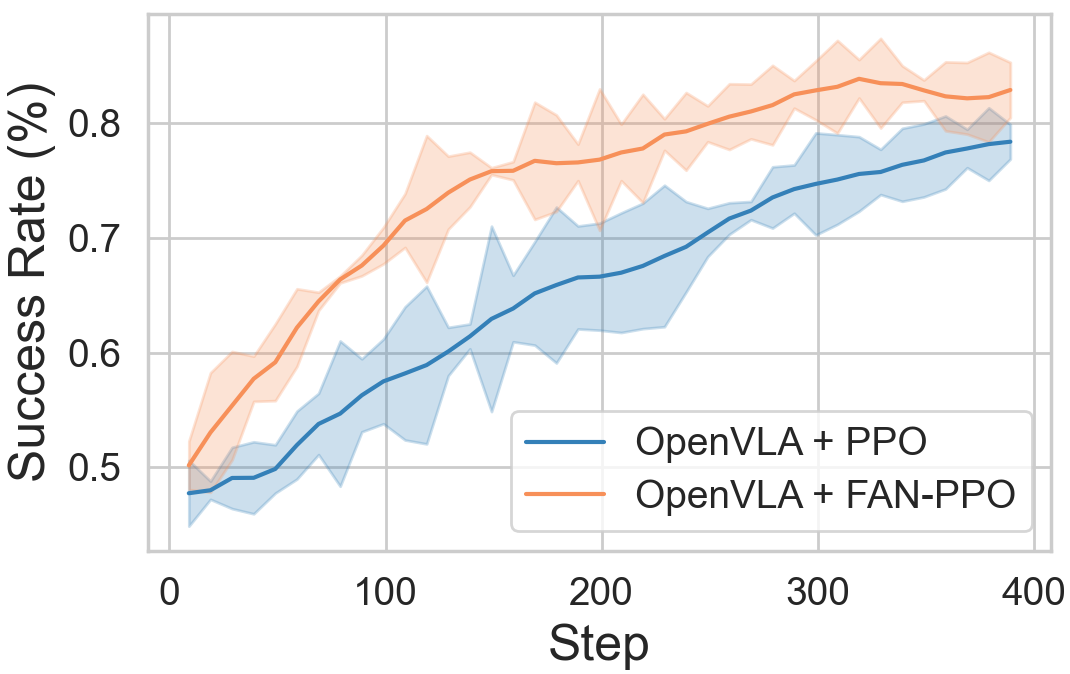}}

    \caption{Training curves of rollout and evaluation success rates during RFT of OpenVLA. Our FAN-guided regularization significantly improved the sample efficiency of OpenVLA.
    }
    \label{fig:openvla_rl}
\end{figure}

\begin{figure}[!t]
    \centering
    \subfloat[Rollout on the training set.\label{fig:ms_oft_rl}]{
        \includegraphics[width=0.47\linewidth]{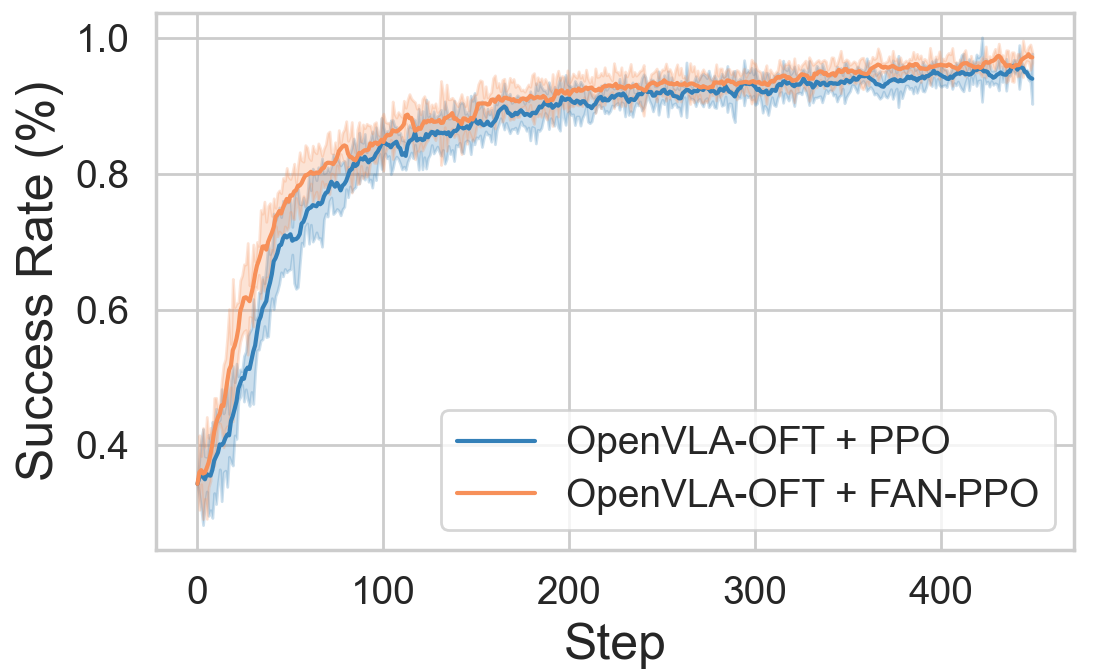}}
    \hfill
    \subfloat[Evaluation on the test set.\label{fig:ms_oft_rl_ood}]{
        \includegraphics[width=0.47\linewidth]{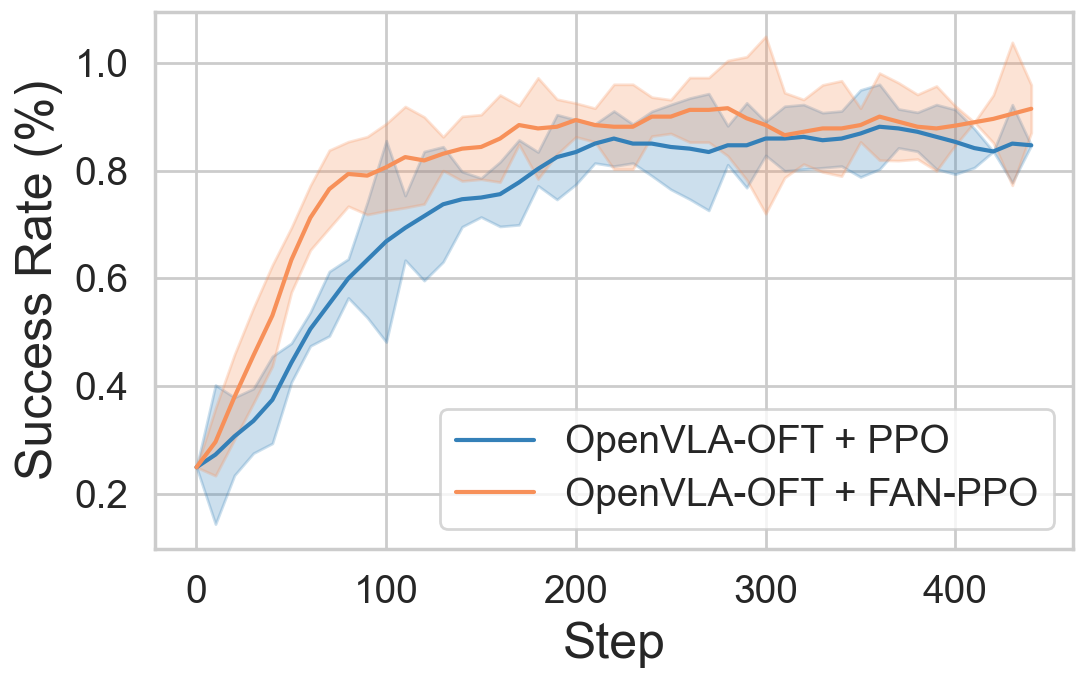}}

    \caption{Training curves of rollout and evaluation success rates during RFT of OpenVLA-OFT. Our FAN-guided regularization significantly improved the sample efficiency in evaluation and obtained better performance in training process of OpenVLA-OFT.
    }
    \label{fig:oft_rl}
    \vspace{-10pt}
\end{figure}

\tref{tab:ms_rl} summarizes the quantitative results of RFT on ManiSkill. 
It can be seen that applying FAN-guided regularization to both OpenVLA and OpenVLA-OFT consistently improves performance across in-distribution and all OOD settings. 
Notably, the improvement is most pronounced in the execution OOD setting, where the success rate increases by more than 10\%. 
As shown in \fref{fig:rl_ood}, which provides a more detailed breakdown across different OOD tasks, FAN-guided regularization yields consistent performance gains similar to results in the SFT experiments in \sref{subsec:sft}.

\fref{fig:openvla_rl} presents the rollout and evaluation success rates during RFT of OpenVLA. 
Clearly, FAN-guided regularization significantly accelerates convergence, effectively improving sample efficiency, achieving a 90\% success rate with only about one-third of the training steps required by the baseline. 
A similar trend can be observed under OOD evaluation, where our method maintains faster improvement and higher final performance. 
\fref{fig:oft_rl} presents the results on OpenVLA-OFT. 
Although less pronounced compared to OpenVLA, FAN-guided regularization still yields distinct performance gains throughout training, with more evident advantage in OOD settings.
More detailed quantitative results are provided in Appendix \tref{tab_sup:sample1}-\ref{tab_sup:sample4}.

\subsection{Real-World Experiment}

\begin{figure}[!t]
    \centering
    \subfloat[JAKA robot platform]{
        \label{subfig:jaka_robot}
        \includegraphics[width=0.42\linewidth]{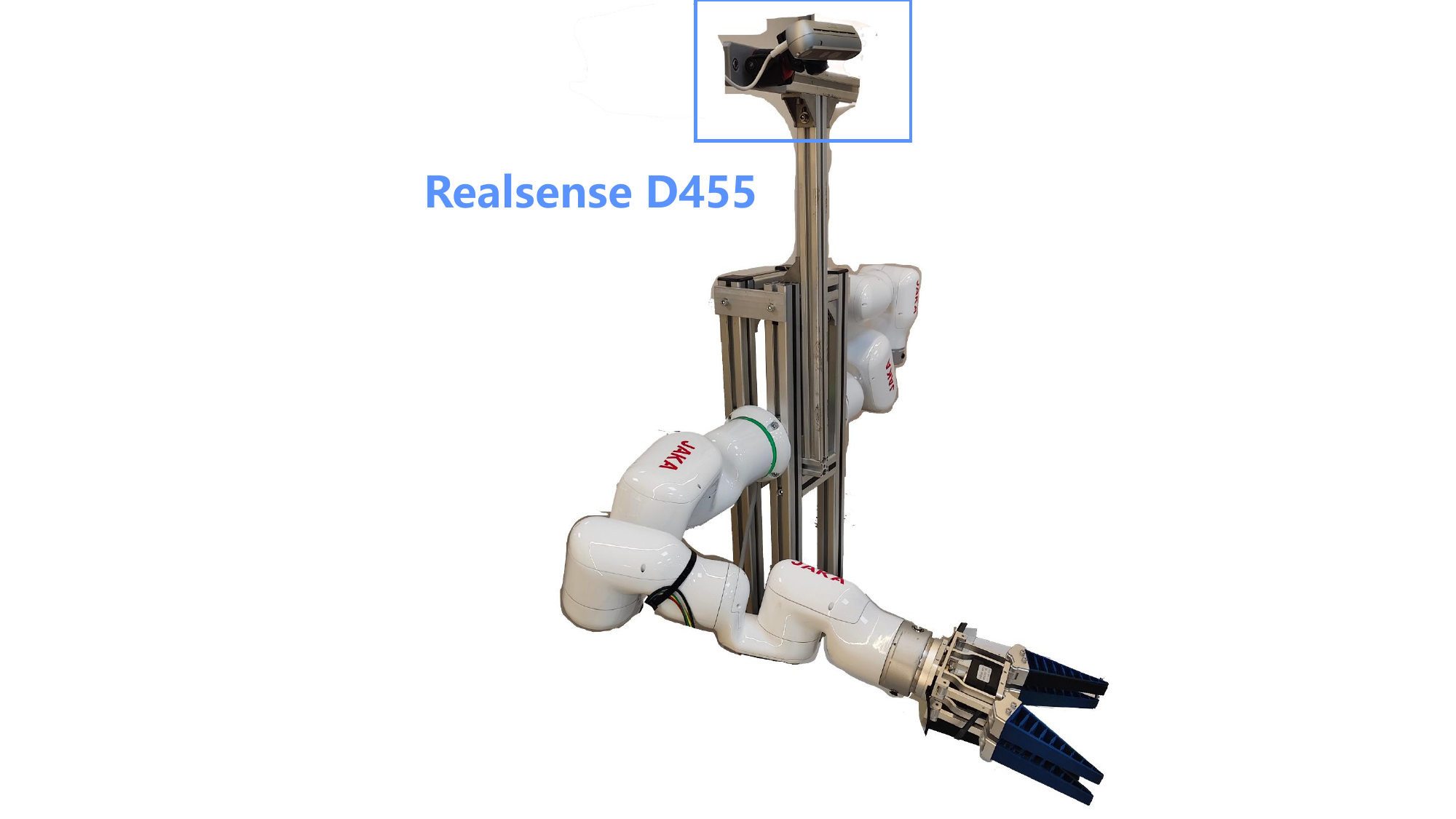}}
    \hspace{0.05\linewidth}
    \subfloat[Placing the object into a box from unseen spatial positions]{  
        \label{subfig:frenchs_placing}
        \centering
        \includegraphics[width=0.21\linewidth]{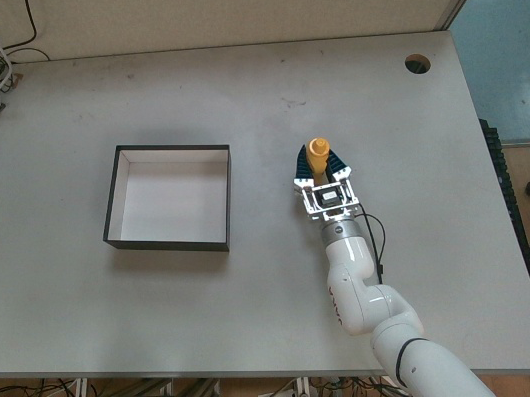}
        \hspace{0.02\linewidth}
        \includegraphics[width=0.21\linewidth]{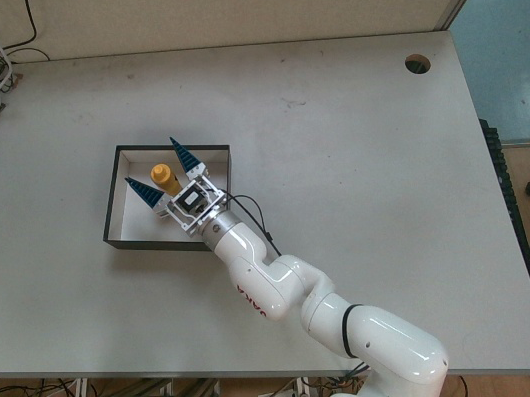}
    }
    \caption{%
    Real-world physical platform. (a) JAKA robot platform; (b) the task of placing the object into a box from unseen spatial positions, to evaluate the model's spatial robustness.
    }
    \label{fig:real_world_example}
\end{figure}

To evaluate the effectiveness of our method in real-world settings, we built a physical robotic platform for comparative experiments, as illustrated in \fref{fig:real_world_example}. 
The platform consists of a 7‑DoF JAKA manipulator equipped with a parallel gripper and a fixed Intel RealSense D455 camera that provides third-person RGB images.

\begin{table}[!t]
    \centering
    \small
    \caption{Real-world evaluation results (Successes / Total Trials).}
    \label{tab:real_world}
    \renewcommand{\arraystretch}{1.0}
    \begin{tabular*}{\linewidth}{@{\extracolsep{\fill}}lcccc@{}}
        \toprule
        \textbf{Method} & \textbf{Task-1} & \textbf{Task-2} & \textbf{Task-3} & \textbf{Task-4} \\
        \midrule
        OpenVLA\cite{kim2024OpenVLA} & 19/30 & 7/30 & 7/30 & 1/30 \\
        OpenVLA + FAN-SFT & 22/30 & 12/30 & 17/30 & 7/30 \\
        \bottomrule
    \end{tabular*}
\end{table}

We evaluate the OpenVLA models trained with SFT and the proposed FAN-SFT.
A total of 150 demonstration trajectories (50 per object) are collected for training. 
We define four tasks for evaluation: 
Task-1 corresponds to the standard setting, consistent with the demonstrations, e.g., placing the target object into a designated box, as shown in \fref{subfig:frenchs_placing}; 
Task-2 perturbs the initial pose of the object; 
Task-3 perturbs the manipulator’s initial pose; 
and Task-4 perturbs the box’s initial position. 
Each task is evaluated over 30 independent trials. The results in \tref{tab:real_world} show that our method exhibits stronger robustness in the presence of spatial perturbations. 
More details are provided in Appendix.

\section{Conclusion}
In this work, we addressed the fundamental mismatch between VLA finetuning methods and the intrinsic tolerance of physical actions, which is a property we term FAN. We first demonstrated empirically that standard methods, which ignore this property, produce ``spiky" policy distributions that fail to generalize. To remedy this, we introduced a FAN-guided regularizer, applicable to both SFT and RFT. This regularizer explicitly shapes the policy's geometric structure, guiding it to approximate a target Gaussian. Our extensive experiments demonstrated that by actively shaping the policy distribution, VLA models achieve significant improvement in {sample efficiency and OOD generalization}. 
Overall, our findings indicate that controlling policy geometry constitutes a principled and highly effective direction for developing robust and generalizable robotic agents.

\section*{Acknowledgements}
This work was supported by the Science and Technology Innovation (STI) 2030--Major Project under Grant 2022ZD0208700 and Shanghai Pilot Industry Innovation and Development Project (Artificial Intelligence Program) under Grant 2025-GZL-RGZN-01033. 

{
    \small
    \bibliographystyle{ieeenat_fullname}
    \bibliography{main}
}

\clearpage
\setcounter{page}{1}
\maketitlesupplementary

\tableofcontents
\newpage

\addtocontents{toc}{\protect\setcounter{tocdepth}{2}}
\section{Proof of Proposition 1}

\begin{proof}[Proof of Proposition 1]
    Because $\pi(\cdot|s, l)$ is a probability distribution, given any $s \in S$, it satisfies $\sum_{a \in A} \pi(a|s, l)=1$. We also add this constraint into the objective. The objective can be modified as
    \begin{equation*}
        \begin{aligned}
            \max_{\pi} \quad &  \mathbb{E}_{\pi_t}\left[ \frac{\pi(a|s, l)}{\pi_t(a|s, l)} A^{\pi_t}(s,a, l) \right]  - \alpha D_{\text{KL}}(\pi(\cdot|s, l) \| \mathcal{N}(\cdot|\mu(s),\Sigma)) \\
            \text{s.t.} \quad & D_{\text{KL}}(\pi(\cdot|s, l) \| \pi_t(\cdot|s, l)) \leq \epsilon \text{ and } \sum_{a} \pi(a|s, l) = 1
        \end{aligned}
    \end{equation*}
    Using method of Lagrange multiplier for the objective, we have
    \begin{align*}
        \mathcal{L}(\pi, \beta, \lambda)
        = & \sum_{a \in A} \pi(a|s, l) A^{\pi_t}(s,a, l) - \alpha \sum_{a \in A} \pi(a|s, l) \log \frac{\pi(a|s, l)}{\mathcal{N}(a|\mu(s),\Sigma)} - \beta \left( \sum_{a \in A} \pi(a|s, l) \log \frac{\pi(a|s, l)}{\pi_t(a|s, l)} - \epsilon \right) \\
        &  - \lambda (\sum_{a} \pi(a|s, l) - 1) \\
        = & \sum_{a \in A} \pi(a|s, l) A^{\pi_t}(s,a, l) - \alpha \sum_{a \in A} \pi(a|s, l) \log \pi(a|s, l)  + \alpha  \sum_{a \in A} \pi(a|s, l) \log \mathcal{N}(a|\mu(s),\Sigma) \\
        & - \beta \sum_{a \in A} \pi(a|s, l) \log \pi(a|s, l)  + \beta\sum_{a \in A} \pi(a|s, l) \log \pi_t(a|s, l)  - \lambda (\sum_{a} \pi(a|s, l) - 1) + \beta \epsilon \\
        = & \sum_{a \in A} \pi(a|s, l) A^{\pi_t}(s,a, l) - (\alpha + \beta) \sum_{a \in A} \pi(a|s, l) \log \pi(a|s, l) + \alpha  \sum_{a \in A} \pi(a|s, l) \log \mathcal{N}(a|\mu(s),\Sigma) \\
        & + \beta\sum_{a \in A} \pi(a|s, l) \log \pi_t(a|s, l) - \lambda (\sum_{a} \pi(a|s, l) - 1) + \beta \epsilon .
    \end{align*}
    From the property of the method of Lagrange multiplier, $\mathcal{L}(\pi, \beta, \lambda)$ is a convex function for $\pi$. To find the optimal policy $\pi^*$, we solve for the critical point of $\mathcal{L}(\pi, \beta, \lambda)$. Given any $a$, we calculate the partial derivative of $\mathcal{L}(\pi, \beta, \lambda)$, then we have
    \begin{align*}
        \frac{\partial \mathcal{L}(\pi, \beta, \lambda)}{ \partial \pi(a|s, l)} & =  A^{\pi_t}(s,a, l) - (\alpha + \beta) (\log \pi(a|s, l) + 1)  + \alpha \log \mathcal{N}(a|\mu(s),\Sigma) + \beta\log \pi_t(a|s, l) - \lambda
    \end{align*}
    We then set $\frac{\partial \mathcal{L}(\pi, \beta, \lambda)}{ \partial \pi(a|s, l)} = 0$, and have
    \begin{align*}
        & A^{\pi_t}(s,a, l) - (\alpha + \beta) (\log \pi(a|s, l) + 1)  + \alpha \log \mathcal{N}(a|\mu(s),\Sigma) + \beta\log \pi_t(a|s, l) - \lambda = 0 \\
        \Rightarrow & (\alpha + \beta) (\log \pi(a|s, l) + 1) = A^{\pi_t}(s,a, l)  + \alpha \log \mathcal{N}(a|\mu(s),\Sigma) + \beta\log \pi_t(a|s, l) - \lambda \\
        \Rightarrow & \log \pi(a|s, l)= \frac{A^{\pi_t}(s,a, l)}{\alpha + \beta} + \frac{\alpha}{\alpha + \beta} \log \mathcal{N}(a|\mu(s),\Sigma) + \frac{\beta}{\alpha + \beta} \log \pi_t(a|s, l)  - \frac{\lambda}{\alpha + \beta} - 1 \\
        \Rightarrow & \pi(a|s, l) = \exp\left( \frac{A^{\pi_t}(s,a, l)}{\alpha + \beta}  \right)\mathcal{N}(a|\mu(s),\Sigma)^{\frac{\alpha}{\alpha + \beta}}  \pi_t(a|s, l)^{\frac{\beta}{\alpha + \beta}}  \exp\left( - \frac{\lambda}{\alpha + \beta} - 1 \right) \\
        \Rightarrow & \pi(a|s, l) = \exp\left( \frac{Q^{\pi_t}(s,a, l)}{\alpha + \beta}  \right)\mathcal{N}(a|\mu(s),\Sigma)^{\frac{\alpha}{\alpha + \beta}}  \pi_t(a|s, l)^{\frac{\beta}{\alpha + \beta}}  \exp\left( \frac{-\lambda}{\alpha + \beta} - 1 \right) \exp\left( \frac{-V^{\pi_t}(s, l)}{\alpha + \beta}  \right)
    \end{align*}
    Because the last two terms are not related to action $a$ and the optimal value $\beta^*$ for $\beta$ can be calculated using KKT condition, we can then have the form of the optimal policy $\pi_{t+1}$ as
    \begin{align*}
        \pi_{t+1}(a|s, l) & \propto  \mathcal{N}(a|\mu(s),\Sigma)^{\frac{\alpha}{\alpha + \beta^*}} \pi_t(a|s, l)^{\frac{\beta^*}{\alpha + \beta^*}} \exp\left( \frac{Q^{\pi_t}(s,a, l)}{\alpha + \beta^*}\right).
    \end{align*} 
    The proof is completed.
\end{proof}

\section{Additional Experiments on Supervised Finetuning}
\subsection{Implementation Details}

For a fair comparison, we adopt the hyperparameter configurations reported in the original OpenVLA and OpenVLA-OFT implementations for each corresponding setting. 
Our proposed method, FAN-SFT, builds upon this established setup and introduces an additional regularization loss, governed by the coefficient \(\alpha\).

Table~\ref{tab_sup:hyperparams_sft1}, Table~\ref{tab_sup:hyperparams_sft2}, and Table~\ref{tab_sup:hyperparams_sft3} provide a comprehensive summary of the hyperparameters used in our experiments. 
Specifically, they respectively detail the SFT hyperparameters for OpenVLA on the ManiSkill benchmark, for OpenVLA on the LIBERO benchmark, and for OpenVLA-OFT on the LIBERO benchmark, including the differences between the baseline configuration and our FAN-SFT specific addition.

We evaluate our method on a diverse set of out-of-distribution (OOD) task variants on ManiSkill, as in \cite{liu2025What}, 
which can be grouped into three categories: 
vision, semantic, and execution.
Table~\ref{tab_sup:ood_tasks} summarizes the definition of each variant.

\begin{table*}[htbp]
    \centering
    \caption{Hyperparameters for SFT OpenVLA on the ManiSkill benchmark. We adopt the configuration from RL4VLA for our baseline and extend it with a regularization coefficient $\alpha$ for our method (FAN-SFT).}
    \label{tab_sup:hyperparams_sft1}
    \renewcommand{\arraystretch}{1.2}
    \begin{tabular}{@{}ll@{}}
        \toprule
        \multicolumn{2}{@{}l}{\textbf{Base Configuration (adopted from RL4VLA)}} \\
        \midrule
        \# GPUs & 4 $\times$ NVIDIA A100 (80GB VRAM) \\
        learning rate (LR) & 5e-4 \\
        total batch size & 40 (10 per GPU) \\
        grad accumulation steps & 1 \\
        input images & 1 third-person camera \\
        input image size & 224 $\times$ 224 px \\
        LoRA rank & 32 \\
        \midrule
        \multicolumn{2}{@{}l}{\textbf{Addition for Our Method (FAN-SFT)}} \\
        \midrule
        regularization coeff. $\alpha$ &
            0.05 \\
        \bottomrule
    \end{tabular}
\end{table*}

\begin{table*}[htbp]
    \centering
    \caption{Hyperparameters for SFT OpenVLA on the LIBERO benchmark. We adopt the configuration from OpenVLA for our baseline and extend it with a regularization coefficient $\alpha$ for our method (FAN-SFT).}
    \label{tab_sup:hyperparams_sft2}
    \renewcommand{\arraystretch}{1.2}
    \begin{tabular}{@{}ll@{}}
        \toprule
        \multicolumn{2}{@{}l}{\textbf{Base Configuration (adopted from OpenVLA)}} \\
        \midrule
        \# GPUs & 2 $\times$ NVIDIA A100 (80GB VRAM) \\
        learning rate (LR) & 5e-4 \\
        total batch size & 48 (24 per GPU) \\
        grad accumulation steps & 1 \\
        input images & 1 third-person camera \\
        input image size & 224 $\times$ 224 px \\
        LoRA rank & 32 \\
        \midrule
        \multicolumn{2}{@{}l}{\textbf{Addition for Our Method (FAN-SFT)}} \\
        \midrule
        regularization coeff. $\alpha$ &
        \begin{tabular}[t]{@{}l@{}}
            {0.01}
        \end{tabular} \\
        \bottomrule
    \end{tabular}
\end{table*}

\begin{table*}[htbp]
    \centering
    \caption{Hyperparameters for SFT OpenVLA-OFT on the LIBERO benchmark. We adopt the configuration from OpenVLA-OFT for our baseline and extend it with a regularization coefficient $\alpha$ for our method (FAN-SFT).}
    \label{tab_sup:hyperparams_sft3}
    \renewcommand{\arraystretch}{1.2}
    \begin{tabular}{@{}ll@{}}
        \toprule
        \multicolumn{2}{@{}l}{\textbf{Base Configuration (adopted from OpenVLA-OFT)}} \\
        \midrule
        \# GPUs & 4 $\times$ NVIDIA A100 (80GB VRAM) \\
        learning rate (LR) & 5e-4 \\
        total batch size & 32 (8 per GPU) \\
        grad accumulation steps & 2 \\
        input images & 1 third-person camera, 1 wrist-mounted camera \\
        input image size & 224 $\times$ 224 px \\
        LoRA rank & 32 \\
        action chunk size & 8 steps (predict 8, execute all 8 open-loop) \\
        use proprio (robot state) & yes \\
        use FiLM & no \\
        image augmentations & 90\% random crops, color jitter: \\
        & \;\; \texttt{random\_resized\_crop=dict(scale=[0.9, 0.9], ratio=[1.0, 1.0])} \\
        & \;\; \texttt{random\_brightness=[0.2]} \\
        & \;\; \texttt{random\_contrast=[0.8, 1.2]} \\
        & \;\; \texttt{random\_saturation=[0.8, 1.2]} \\
        & \;\; \texttt{random\_hue=[0.05]} \\
        \midrule
        \multicolumn{2}{@{}l}{\textbf{Addition for Our Method (FAN-SFT)}} \\
        \midrule
        regularization coeff. $\alpha$ &
        0.05 \\
        \bottomrule
    \end{tabular}
\end{table*}

\begin{table}[htbp]
    \centering
    \small
    \caption{
    Definitions of the OOD task variants.
    The first five variants correspond to vision OOD, the middle eight to semantic OOD, 
    and the last three to execution OOD.
    }
    \label{tab_sup:ood_tasks}
    \renewcommand{\arraystretch}{1.0}
    \begin{tabular}{ll}
        \toprule
        \textbf{Variant} & \textbf{Definition} \\
        \midrule
        \multicolumn{2}{l}{\textbf{Vision}} \\
        \midrule
        Table      & unseen table \\
        Texture-w  & dynamic texture (weak) \\
        Texture-s  & dynamic texture (strong) \\
        Noise-w    & dynamic noise (weak) \\
        Noise-s    & dynamic noise (strong) \\
        \midrule
        \multicolumn{2}{l}{\textbf{Semantic}} \\
        \midrule
        Obj.        & unseen objects \\
        Recep.      & unseen receptacles \\
        Instruct   & unseen instructions \\
        M-obj. (IND) & multi-object (both seen) \\
        M-obj. (OOD) & multi-object (both unseen) \\
        Disturb Recep. & distractive receptacle \\
        M-Recep.    & multi-receptacle (both unseen) \\
        Obj. Pos.    & unseen position (object \& receptacle) \\
        \midrule
        \multicolumn{2}{l}{\textbf{Execution}} \\
        \midrule
        Robot Pose & unseen robot initial pose \\
        Obj. Rep. & mid-episode object reposition \\
        \bottomrule
    \end{tabular}
\end{table}

\clearpage
\subsection{Additional Results}

\paragraph{ManiSkill.}

For the experiments on the ManiSkill benchmark, we reference the RL4VLA results directly from the original paper. Since the SFT dataset used in RL4VLA is not publicly available, we reproduce its setting by collecting an SFT dataset of the same size in the same environments using a motion planner. Based on this reproduced setup, we obtain the baseline and further implement our proposed OpenVLA + FAN-SFT method.

Detailed results regarding Figure~\ref{fig:ms_ood} are provided in Table~\ref{tab_sup:SFT_combined_results}, while Table~\ref{tab_sup:ms_sft} presents a comprehensive comparison against the baselines. 
As illustrated in Table~\ref{tab_sup:ms_sft}, our proposed OpenVLA + FAN-SFT significantly outperforms the OpenVLA + SFT baseline, achieving a substantial gain of $11.7\%$ in in-distribution tasks and an average improvement of $5.2\%$ across out-of-distribution (OOD) tasks. 
Crucially, this enhancement is consistent across all generalization axes: we observe performance boosts of $5.1\%$, $6.1\%$, and $4.4\%$ in Vision, Semantic, and Execution categories, respectively. The breakdown in Table~\ref{tab_sup:SFT_combined_results} further highlights the robustness of our method in challenging environments. Notably, FAN-SFT exhibits remarkable efficacy in tasks requiring complex reasoning and noise resilience, such as `M-Obj. (OOD)' ($+9.3\%$), `Disturb Recep.' ($+7.8\%$), and `Noise-s' ($+7.2\%$).

\begin{table*}[htbp]
    \centering
    \small
    \caption{Comparison of supervised fine-tuning (SFT) results on the ManiSkill benchmark.
    Values denote success rates (\%).
    Results marked with \textsuperscript{†} are taken directly from the original RL4VLA paper.
    Since the SFT dataset used in RL4VLA is not publicly available, we reproduce its setting by collecting an SFT dataset of the same size in the same environments using a motion planner.
    Based on this reproduced setup, we obtain the OpenVLA + SFT baseline as our reproduction of RL4VLA, and further implement our proposed OpenVLA + FAN-SFT method.
    }  %
    \label{tab_sup:ms_sft}
    \renewcommand{\arraystretch}{1.0}
    \begin{tabular}{ll c cccc}
    \toprule
    & \multirow{2}{*}{\textbf{Method}}
    & \multirow{2}{*}{\textbf{In-Distribution}}
    & \multicolumn{4}{c}{\textbf{Out-of-Distribution}} \\
    \cmidrule(r){4-7}
    &
    &
    & \textbf{Vision} & \textbf{Semantic} & \textbf{Execution} & \textbf{Avg.} \\
    \midrule
    & RL4VLA\textsuperscript{†} \cite{liu2025What} & 88.5 & 74.0 & 61.8 & 46.2 & 60.7 \\
    \midrule
    & OpenVLA + SFT \cite{kim2024OpenVLA} & 78.1 $\pm$ 3.1  & 76.6 $\pm$ 1.9 & 57.4 $\pm$ 0.9 & 40.4 $\pm$ 0.8 & 58.1 \\
    & OpenVLA + FAN-SFT (Ours) & 89.8 $\pm$ 0.8 & 81.7 $\pm$ 1.1   & 63.5 $\pm$ 1.5   & 44.8 $\pm$ 0.5 & 63.3 \\
    \rowcolor{lightgray}
    & $\triangle$ Improvement & +11.7 & +5.1 & +6.1 & +4.4 & +5.2 \\
    \bottomrule
    \end{tabular}
\end{table*}

\begin{table*}[htbp]
    \centering
    \caption{Detailed SFT results on the ManiSkill benchmark (OOD tasks).}
    \label{tab_sup:SFT_combined_results}
    \begin{tabular}{@{}ll|c|ccc@{}}
        \toprule
        & \textbf{Task} 
        & \textbf{RL4VLA(SFT)\textsuperscript{†}} 
        & \textbf{OpenVLA+SFT} 
        & \textbf{OpenVLA+FAN-SFT} 
        & \textbf{$\Delta$} \\
        \midrule

        \multirow{5}{*}{\rotatebox{90}{\textbf{Vision}}}
         & Table        & 86.5 & 80.6 $\pm$ 5.5 & 84.4 $\pm$ 2.9 & +3.8 \\
         & Texture-w    & 79.7 & 80.0 $\pm$ 3.4 & 85.9 $\pm$ 3.1 & +5.9 \\
         & Texture-s    & 63.0 & 72.5 $\pm$ 5.8 & 73.4 $\pm$ 4.5 & +0.9 \\
         & Noise-w      & 77.1 & 81.3 $\pm$ 4.0 & 87.2 $\pm$ 1.8 & +5.9 \\
         & Noise-s      & 63.5 & 66.9 $\pm$ 2.3 & 74.1 $\pm$ 3.2 & +7.2 \\
        \cmidrule(l){1-6}

        \multirow{8}{*}{\rotatebox{90}{\textbf{Semantic}}}
         & Obj.           & 53.6 & 42.5 $\pm$ 4.1 & 49.4 $\pm$ 4.8 & +6.9 \\
         & Recep.         & 68.2 & 70.6 $\pm$ 4.5 & 77.5 $\pm$ 3.6 & +6.9 \\
         & Instruct       & 89.3 & 78.4 $\pm$ 4.4 & 85.0 $\pm$ 1.5 & +6.6 \\
         & M-Obj. (IND)   & 69.3 & 68.1 $\pm$ 1.2 & 71.9 $\pm$ 2.3 & +3.8 \\
         & M-Obj. (OOD)   & 35.9 & 26.6 $\pm$ 4.5 & 35.9 $\pm$ 4.3 & +9.3 \\
         & Disturb Recep. & 72.4 & 69.4 $\pm$ 4.6 & 77.2 $\pm$ 4.1 & +7.8 \\
         & M-Recep.       & 49.0 & 47.8 $\pm$ 4.3 & 48.0 $\pm$ 4.2 & +0.2 \\
        \cmidrule(l){1-6}

        \multirow{4}{*}{\rotatebox{90}{\textbf{Execution}}}
         & Obj. Pos.    & 71.4 & 63.8 $\pm$ 4.8 & 66.9 $\pm$ 4.2 & +3.1 \\
         & Robot Pose   & 34.9 & 34.1 $\pm$ 4.3 & 35.6 $\pm$ 3.2 & +1.5 \\
         & Obj. Rep.    & 32.3 & 25.0 $\pm$ 3.3 & 31.6 $\pm$ 1.8 & +6.6 \\[2mm]
        \bottomrule
    \end{tabular}
\end{table*}

\paragraph{LIBERO.}

Table~\ref{tab_sup:libero_sft_ind} reports the performance comparison on the LIBERO benchmark.
With our FAN-guided regularization, the proposed method consistently improves performance across all LIBERO suites, yielding an average gain of 2.4 percentage points.
In terms of overall average success rate, our approach reaches \(96.9\%\), surpassing the latest state-of-the-art method UniVLA by \(1.7\%\).

Given that existing methods have achieved near-saturated performance on standard In-Distribution (IND) LIBERO tasks (with success rates exceeding \(90\%\), as shown in Table~\ref{tab_sup:libero_sft_ind}), relying solely on IND metrics makes it difficult to distinguish the true robustness of the policies.
To address this and rigorously evaluate spatial generalization, we introduce positional perturbations by shifting the target placement location (i.e., the plate) along the \(x\) and \(y\) axes.
Qualitative visualizations of these tasks are presented in Figure~\ref{fig_sup:libero_tasks_xy}.
As observed from the rollout trajectories, the baseline OpenVLA + SFT policy (top rows) frequently fails to adapt to these spatial shifts, often executing motions biased towards the original training distribution, which we attribute to the single-correct-bin constraint imposed during vanilla SFT training.
In contrast, our method OpenVLA + FAN-SFT (bottom rows) explicitly incorporates a feasible action neighborhood during training, which effectively broadens the policy's admissible action space and improves its ability to adapt to variations in the target location.
As a result, it exhibits stronger spatial generalization and is able to accurately place the object onto the perturbed plate across all tasks.

\begin{table*}[htbp]
    \centering

    \caption{Detailed SFT results on the LIBERO Benchmark.}
    \label{tab_sup:libero_sft_ind}
    \renewcommand{\arraystretch}{1.0}
    \setlength{\tabcolsep}{10pt}
    \begin{tabular}{l cccc | c}
    \toprule
    \textbf{Method} & \textbf{Spatial} & \textbf{Object} & \textbf{Goal} & \textbf{Long} & \textbf{Avg.} \\
    \midrule
    Octo\cite{ghosh2024Octo} & 78.9 & 85.7 & 84.6 & 51.1 & 75.1 \\
    OpenVLA\cite{kim2024OpenVLA} & 84.7 & 88.4 & 79.2 & 53.7 & 76.5 \\
    WorldVLA\cite{cen2025WorldVLA} & 87.6 & 96.2 & 83.4 & 60.0 & 81.8 \\
    TriVLA\cite{liu2025TriVLA} & 91.2 & 93.8 & 89.8 & 73.2 & 87.0 \\  
    UniVLA\cite{bu2025UniVLA} & 96.5 & \textbf{96.8} & 95.6 & 92.0 & 95.2 \\  
    \midrule
    OpenVLA-OFT\cite{kim2025FineTuning} & 95.2 & 94.2 & 95.2 & 93.2 & 94.5 \\
    \textbf{OpenVLA-OFT + FAN-SFT (Ours)} & \textbf{98.8} & 96.6 & \textbf{97.0} & \textbf{95.2} & \textbf{96.9} \\
    \rowcolor{lightgray}
    $\triangle$ Improvement & +3.6 & +2.4 & +1.8 & +2.0 & +2.4 \\
    \bottomrule
    \end{tabular}
\end{table*}

\begin{figure}[htbp]
    \centering

    \begin{subfigure}{0.9\textwidth}
        \centering
        \includegraphics[width=\textwidth]{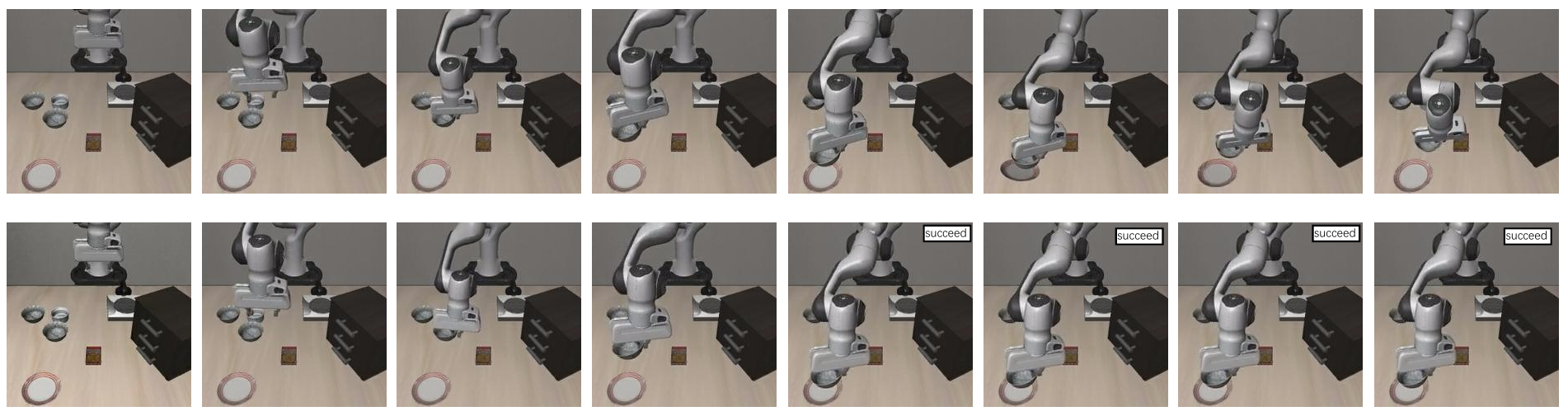}
        \caption{
        Task 1 (\(x\)-axis disturbance on the plate): pick up the black bowl between the plate and the ramekin and place it on the (perturbed) plate.
        The baseline OpenVLA + SFT policy (top) fails, whereas our OpenVLA + FAN-SFT policy (bottom) succeeds.
        }
        \label{fig_sup:libero_task1}
    \end{subfigure}

    \vspace{0.6em}

    \begin{subfigure}{0.9\textwidth}
        \centering
        \includegraphics[width=\textwidth]{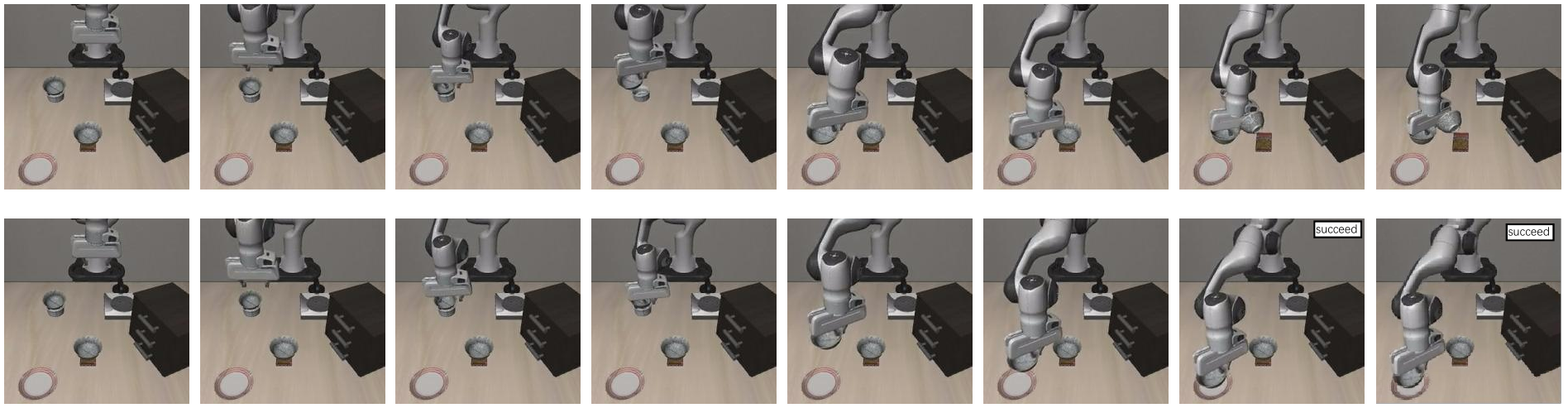}
        \caption{
        Task 2 (\(x\)-axis disturbance on the plate): pick up the black bowl on the ramekin and place it on the (perturbed) plate.
        The baseline OpenVLA + SFT policy (top) fails, whereas our OpenVLA + FAN-SFT policy (bottom) succeeds.
        }
        \label{fig_sup:libero_task2}
    \end{subfigure}

    \vspace{0.6em}

    \begin{subfigure}{0.9\textwidth}
        \centering
        \includegraphics[width=\textwidth]{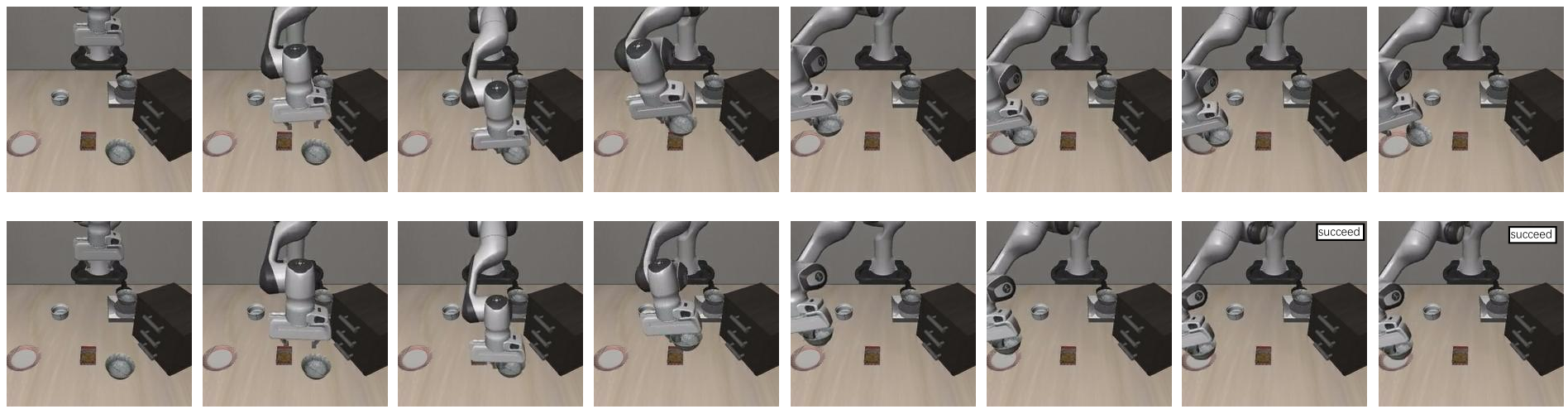}
        \caption{
        Task 3 (\(y\)-axis disturbance on the plate): pick up the black bowl next to the cookie box and place it on the (perturbed) plate.
        The baseline OpenVLA + SFT policy (top) fails, whereas our OpenVLA + FAN-SFT policy (bottom) succeeds.
        }
        \label{fig_sup:libero_task3}
    \end{subfigure}

    \vspace{0.6em}

    \begin{subfigure}{0.9\textwidth}
        \centering
        \includegraphics[width=\textwidth]{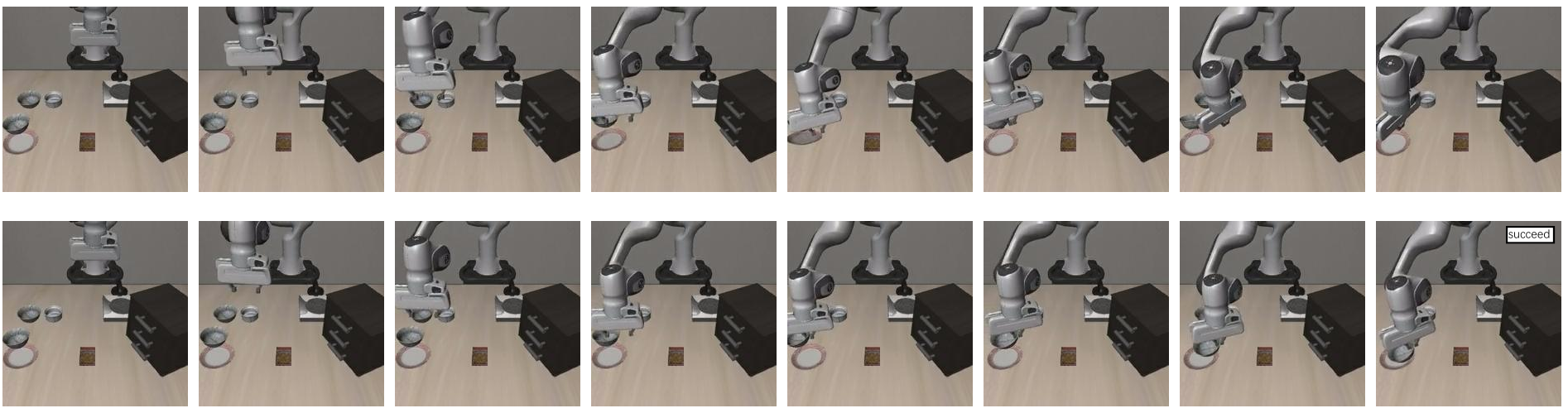}
        \caption{
        Task 4 (\(y\)-axis disturbance on the plate): pick up the black bowl next to the plate and place it on the (perturbed) plate.
        The baseline OpenVLA + SFT policy (top) fails, whereas our OpenVLA + FAN-SFT policy (bottom) succeeds.
        }
        \label{fig_sup:libero_task4}
    \end{subfigure}

    \caption{
    Simulation rollouts on four manipulation tasks from the LIBERO-Spatial benchmark under disturbances applied to the plate (i.e., the target placement location).
    For Tasks 1 and 2, the plate is perturbed along the \(x\)-axis; for Tasks 3 and 4, the plate is perturbed along the \(y\)-axis.
    In each subfigure, the top row shows trajectories generated by the baseline OpenVLA + SFT policy,
    and the bottom row shows trajectories generated by our proposed OpenVLA + FAN-SFT policy,
    which exhibits stronger robustness to target-location disturbances.
    }
    \label{fig_sup:libero_tasks_xy}
\end{figure}

\clearpage
\subsection{Parameter Sensitivity Analysis}\label{supp:parameter}

We evaluate the sensitivity of our proposed OpenVLA + FAN-SFT method to the regularization coefficient \(\alpha\) on the ManiSkill benchmark, as shown in Table~\ref{tab_sup:ms_sft_hyperparameter}. 
The empirical results demonstrate that, across a broad range of moderate values, introducing our FAN-guided regularization consistently leads to performance improvements over the baseline. 
However, when the regularization coefficient becomes excessively large, it starts to hinder the training process and degrades performance. 
In particular, we adopt a regularization coefficient of \(\alpha = 0.05\) for our SFT experiments, which achieves near-optimal performance under this hyperparameter sweep.

\begin{table*}[htbp]
    \centering
    \small
    \caption{Sensitivity of OpenVLA + FAN-SFT to the regularization coefficient $\alpha$ on ManiSkill.}  
    \label{tab_sup:ms_sft_hyperparameter}
    \renewcommand{\arraystretch}{1.0}
    \begin{tabular}{ll c cccc}
    \toprule
    & \multirow{2}{*}{\textbf{Regularization coeff. $\alpha$}}
    & \multirow{2}{*}{\textbf{In-Distribution}}
    & \multicolumn{4}{c}{\textbf{Out-of-Distribution}} \\
    \cmidrule(r){4-7}
    &
    &
    & \textbf{Vision} & \textbf{Semantic} & \textbf{Execution} & \textbf{Avg.} \\
    \midrule
    & 0.0 & 78.1 & 76.3 & 57.6 & 41.0 & 58.3 \\
    & 1e-4 & 81.8 & 76.2 & 58.0 & 44.6 & 59.6 \\
    & 1e-3 & 87.0 & 75.6 & 60.9 & 44.2 & 60.2 \\
    & 1e-2 & 89.6 & 77.5 & 62.9 & 46.5 & 62.3 \\
    & 2e-2 & 88.4 & 77.5 & 62.9 & \textbf{49.1} & 63.2 \\
    & 5e-2 & \textbf{89.8} & \textbf{81.7} & \textbf{63.5} & 44.8 & \textbf{63.3} \\
    & 0.1 & 82.8 & 76.0 & 60.8 & 45.0 & 60.6 \\
    & 1.0 & 83.9 & 78.8 & 60.4 & 43.6 & 60.9 \\
    & 2.0 & 80.7 & 76.5 & 59.3 & 45.0 & 60.3 \\
    \bottomrule
    \end{tabular}
\end{table*}

We further compare our proposed OpenVLA + FAN-SFT method against label smoothing on the ManiSkill benchmark, as shown in Table~\ref{tab_sup:ls}. 
The empirical results indicate that label smoothing yields only modest benefits when the smoothing factor $\epsilon$ is properly chosen. 
Among the evaluated settings, $\epsilon = 0.05$ gives the strongest performance, but still remains clearly inferior to FAN across both in-distribution and out-of-distribution evaluations. 
As the smoothing factor further increases, the performance degrades, suggesting that the unstructured probability spreading introduced by label smoothing is suboptimal for modeling robotic actions. 
By contrast, FAN imposes a more structured form of regularization that better respects the local geometry of the action space, and therefore provides more effective and stable improvements.

\begin{table*}[htbp]
    \centering
    \small
    \caption{Comparison to label smoothing on ManiSkill.}
    \label{tab_sup:ls}
    \renewcommand{\arraystretch}{1.0}
    \begin{tabular}{ll c cccc}
    \toprule
    & \multirow{2}{*}{\textbf{Method}}
    & \multirow{2}{*}{\textbf{In-Distribution}}
    & \multicolumn{4}{c}{\textbf{Out-of-Distribution}} \\
    \cmidrule(r){4-7}
    &
    &
    & \textbf{Vision} & \textbf{Semantic} & \textbf{Execution} & \textbf{Avg.} \\
    \midrule
    & Ori & 78.1 & \underline{76.6} & 57.4 & 40.4 & 58.1 \\
    & \quad + $\epsilon=0.05$ & \underline{82.8} & 75.9 & \underline{62.3} & \underline{42.2} & \underline{60.1} \\
    & \quad + $\epsilon=0.1$ & 81.3 & 69.1 & 60.3 & 39.4 & 56.3 \\
    & \quad + $\epsilon=0.2$ & 79.7 & 68.1 & 51.8 & \underline{42.2} & 54.0 \\
    \midrule
    & Ours & \textbf{89.8} & \textbf{81.7} & \textbf{63.5} & \textbf{44.8} & \textbf{63.3} \\
    \bottomrule
    \end{tabular}
\end{table*}

\clearpage
\section{Additional Experiments on Reinforced Finetuning}

\subsection{Implementation Details}
On the ManiSkill benchmark, we conduct reinforcement fine-tuning (RFT) experiments by applying our FAN-guided regularization to both OpenVLA and OpenVLA-OFT, and comparing the resulting FAN-PPO variants against the standard PPO baselines. 
In addition, for all RFT experiments we use a fixed \(\sigma\) instead of an adaptive \(\sigma\) in order to improve the stability of RFT. 
The detailed hyperparameter configurations for all PPO-based experiments are summarized in Table~\ref{tab_sup:hyperparams_ppo}.

\begin{table*}[htbp]
    \centering
    \caption{Hyperparameter setup for PPO in our experiments. We follow the configuration from the original implementation, including common parameters and those specific to our FAN-PPO.}
    \label{tab_sup:hyperparams_ppo}
    \renewcommand{\arraystretch}{1.2}
    \begin{tabular}{@{}ll@{}}
        \toprule
        \multicolumn{2}{@{}l}{\textbf{Base Configuration}} \\
        \midrule
        \# GPUs & 1 $\times$ NVIDIA A100 (80GB VRAM) \\
        learning rate (LR) & 1e-4 for policy; 3e-3 for value head \\
        mini-batch size & 8 for OpenVLA; 12 for OpenVLA-OFT \\
        grad accumulation steps & 20 \\
        LoRA rank & 32 \\
        input images & 1 third-person camera \\
        input image size & 224 $\times$ 224 px \\
        No.\ of training episodes        & 390 for OpenVLA; 650 for OpenVLA-OFT \\
        sample trajectories per episode  & 64 for OpenVLA; 96 for OpenVLA-OFT \\
        max steps per trajectory         & 80 \\
        training epochs per episode      & 1 \\
        GAE parameter $(\lambda)$        & 0.95 \\
        PPO clipping ratio $(\epsilon)$  & 0.2 \\
        entropy coefficient              & 0.00 \\
        \midrule
        \multicolumn{2}{@{}l}{\textbf{Addition for Our Method (FAN-PPO)}} \\
        \midrule
        regularization  coefficient $\alpha$ &
            1.0 for OpenVLA; 0.1 for OpenVLA-OFT  \\
        standard deviation $\sigma$ & 0.3 for OpenVLA; 0.2 for OpenVLA-OFT \\
        \bottomrule
    \end{tabular}
\end{table*}

\subsection{Additional Results}
Table~\ref{tab_sup:RFT_combined_results} provides a more detailed complement to Figure~\ref{fig:rl_ood}. 
Entries marked with \textsuperscript{†} correspond to RL4VLA results taken directly from the original paper. 
The OpenVLA + PPO baseline represents our reproduction of RL4VLA, upon which we further build our proposed FAN-PPO method.

For OpenVLA, our FAN-PPO consistently improves over the original PPO across all OOD tasks. 
For OpenVLA-OFT, FAN-PPO achieves comparable performance on some tasks and yields substantial gains on others, such as a 23.5\% improvement on 'M-Obj. (IND)' and an 18.3\% improvement on 'Disturb Recep.'. 
Overall, these results demonstrate that, during the RFT stage, our FAN-guided regularization can effectively enhance the model's performance on OOD tasks.

Tables~\ref{tab_sup:sample1}--\ref{tab_sup:sample4} report the number of training steps required by OpenVLA and OpenVLA-OFT to reach different success rates during training. 
For both models, incorporating our FAN-guided regularization consistently reduces the number of training steps needed to achieve the same success rate across all evaluated thresholds, demonstrating that our method effectively improves sample efficiency.

Figures~\ref{fig_sup:ind}--\ref{fig_sup:ObjRep} present qualitative results of OpenVLA under two RFT methods across various ManiSkill tasks. In each case, the top row shows trajectories obtained with the standard PPO, while the bottom row corresponds to our FAN-PPO. These qualitative examples consistently indicate that our method exhibits better generalization.

\begin{table*}[htbp]
    \centering
    \caption{Detailed RFT results on the ManiSkill benchmark (OOD tasks).}
    \label{tab_sup:RFT_combined_results}
    \begin{tabular}{@{}ll|c|ccc|cccc@{}}
        \toprule
        & \textbf{Task} 
        & \textbf{\makecell{RL4VLA\\(RFT)\textsuperscript{†}}}
        & \textbf{\makecell{OpenVLA\\+PPO}} 
        & \textbf{\makecell{OpenVLA\\+FAN-PPO}}
        & \textbf{$\Delta$}
        & \textbf{\makecell{OpenVLA-OFT\\+PPO}}
        & \textbf{\makecell{OpenVLA-OFT\\+FAN-PPO}}
        & \textbf{$\Delta$} \\
        \midrule

        \multirow{5}{*}{\rotatebox{90}{\textbf{Vision}}}
         & Table        & 84.4 & 91.3 $\pm$ 2.8 & 97.1 $\pm$ 0.4 & +5.9 & 89.0 $\pm$ 1.9 & 93.5 $\pm$ 0.9 & +4.5 \\
         & Texture-w    & 83.3 & 85.0 $\pm$ 0.9 & 90.6 $\pm$ 6.1 & +5.7 & 93.1 $\pm$ 3.0 & 92.8 $\pm$ 2.3 & -0.3 \\
         & Texture-s    & 63.0 & 69.2 $\pm$ 2.3 & 72.4 $\pm$ 0.7 & +3.2 & 79.3 $\pm$ 2.4 & 83.3 $\pm$ 5.1 & +4.0\\
         & Noise-w      & 85.4 & 83.0 $\pm$ 1.2 & 91.2 $\pm$ 1.0 & +8.2 & 90.3 $\pm$ 1.1 & 92.8 $\pm$ 0.9 & +2.5 \\
         & Noise-s      & 66.7 & 72.0 $\pm$ 4.6 & 73.7 $\pm$ 6.5 & +1.7 & 73.0 $\pm$ 2.7 & 78.0 $\pm$ 4.3 & +5.0 \\

        \cmidrule(l){1-9}
        \multirow{7}{*}{\rotatebox{90}{\textbf{Semantic}}}
         & Obj.           & 71.4 & 78.3 $\pm$ 3.6 & 85.4 $\pm$ 2.7 & +7.1 & 77.0 $\pm$ 3.7 & 83.6 $\pm$ 1.7 & +6.6 \\
         & Recep.         & 75.0 & 82.2 $\pm$ 4.5 & 90.9 $\pm$ 4.1 & +8.7 & 71.4 $\pm$ 2.8 & 71.2 $\pm$ 6.4 & -0.2 \\
         & Instruct       & 89.1 & 89.8 $\pm$ 2.9 & 95.6 $\pm$ 1.3 & +5.7 & 70.6 $\pm$ 2.0 & 70.7 $\pm$ 0.7 & +0.1 \\
         & M-Obj. (IND)   & 75.0 & 82.6 $\pm$ 1.6 & 90.1 $\pm$ 1.5 & +7.6 & 38.1 $\pm$ 2.6 & 61.6 $\pm$ 1.3 & +23.5 \\
         & M-Obj. (OOD)   & 57.8 & 68.5 $\pm$ 0.7 & 76.3 $\pm$ 1.0 & +7.8 & 41.6 $\pm$ 6.7 & 56.1 $\pm$ 1.0 & +14.5 \\
         & Disturb Recep. & 81.2 & 86.2 $\pm$ 2.6 & 94.3 $\pm$ 1.3 & +8.1 & 26.0 $\pm$ 0.9 & 44.3 $\pm$ 0.9 & +18.3 \\
         & M-Recep.       & 59.9 & 67.2 $\pm$ 1.7 & 74.5 $\pm$ 5.2 & +7.3 & 18.6 $\pm$ 3.1 & 22.5 $\pm$ 3.5 & +3.9 \\

        \cmidrule(l){1-9}
        \multirow{3}{*}{\rotatebox{90}{\textbf{Execution}}}
         & Obj. Pos.    & 80.7 & 86.5 $\pm$ 2.1 & 94.0 $\pm$ 1.8 & +7.6 & 43.9 $\pm$ 4.7 & 54.3 $\pm$ 6.7 & +10.4 \\
         & Robot Pose   & 79.7 & 83.9 $\pm$ 3.0 & 92.3 $\pm$ 1.6 & +8.6 & 64.8 $\pm$ 4.4 & 74.1 $\pm$ 2.9 & +9.3 \\
         & Obj. Rep.    & 74.5 & 87.0 $\pm$ 4.6 & 91.4 $\pm$ 1.7 & +4.4 & 59.0 $\pm$ 2.5 & 72.6 $\pm$ 2.5 & +13.5 \\[2mm]
        \bottomrule
    \end{tabular}
\end{table*}

\begin{table}[!t]
    \centering
    \small
    \caption{
        Training steps required to reach different rollout success rate thresholds on the training subset 
        of ManiSkill with OpenVLA model.
        Columns report the number of training steps needed to achieve
        60\%, 70\%, 80\%, and 90\% success rate (SR), respectively.
        Our method (OpenVLA + FAN-PPO) consistently attains the same
        success rate with fewer steps, demonstrating improved sample efficiency.
    }
    \label{tab_sup:sample1}
    \renewcommand{\arraystretch}{1.0}
    \begin{tabular}{lcccc}
        \toprule
        \textbf{Method} 
        & \textbf{Steps to 60\% SR}
        & \textbf{Steps to 70\% SR}
        & \textbf{Steps to 80\% SR}
        & \textbf{Steps to 90\% SR} \\
        \midrule
        OpenVLA + PPO & 18 & 62 & 133 & 249 \\
        OpenVLA + FAN-PPO            & 18 & \textbf{37} & \textbf{56}  & \textbf{98}  \\
        \bottomrule
    \end{tabular}
\end{table}

\begin{table}[!t]
    \centering
    \small
    \caption{
        Training steps required to reach different evaluation success rate on the test subset 
        of ManiSkill with OpenVLA model.
        Columns report the number of training steps needed to achieve
        55\%, 60\%, 65\%, 70\%, and 75\% success rate (SR), respectively.
        Our method (OpenVLA + FAN-PPO) consistently attains the same
        success rate with fewer steps, demonstrating improved sample efficiency.
    }
    \label{tab_sup:sample2}
    \renewcommand{\arraystretch}{1.0}
    \begin{tabular}{lccccc}
        \toprule
        \textbf{Method} 
        & \textbf{Steps to 55\% SR}
        & \textbf{Steps to 60\% SR}
        & \textbf{Steps to 65\% SR}
        & \textbf{Steps to 70\% SR}
        & \textbf{Steps to 75\% SR} \\
        \midrule
        OpenVLA + PPO            & 109 & 149 & 209 & 279 & 339  \\
        OpenVLA + FAN-PPO & \textbf{29} & \textbf{59} & \textbf{109} & \textbf{129} & \textbf{179} \\
        \bottomrule
    \end{tabular}
\end{table}

\begin{table}[!t]
    \centering
    \small
    \caption{
        Training steps required to reach different rollout success rate thresholds on the training subset 
        of ManiSkill with OpenVLA-OFT model.
        Columns report the number of training steps needed to achieve
        50\%, 60\%, 70\%, 80\%, and 90\% success rate (SR), respectively.
        Our method (OpenVLA-OFT + FAN-PPO) consistently attains the same
        success rate with fewer steps, demonstrating improved sample efficiency.
    }
    \label{tab_sup:sample3}
    \renewcommand{\arraystretch}{1.0}
    \begin{tabular}{lccccc}
        \toprule
        \textbf{Method} 
        & \textbf{Steps to 50\% SR}
        & \textbf{Steps to 60\% SR}
        & \textbf{Steps to 70\% SR}
        & \textbf{Steps to 80\% SR}
        & \textbf{Steps to 90\% SR} \\
        \midrule
        OpenVLA-OFT + PPO & 26 & 35 & 46 & 82 & 186 \\
        OpenVLA-OFT + FAN-PPO            & \textbf{17} & \textbf{24} & \textbf{38} & \textbf{60} & \textbf{151}  \\
        \bottomrule
    \end{tabular}
\end{table}

\begin{table}[!t]
    \centering
    \small
    \caption{
        Training steps required to reach different evaluation success rate on the test subset of ManiSkill  with OpenVLA-OFT model.
        Columns report the number of training steps needed to achieve
        50\%, 60\%, 70\%, 80\%, and 90\% success rate (SR), respectively.
        Our method (OpenVLA-OFT + FAN-PPO) consistently attains the same
        success rate with fewer steps, demonstrating improved sample efficiency.
    }
    \label{tab_sup:sample4}
    \renewcommand{\arraystretch}{1.0}
    \begin{tabular}{lccccc}
        \toprule
        \textbf{Method} 
        & \textbf{Steps to 50\% SR}
        & \textbf{Steps to 60\% SR}
        & \textbf{Steps to 70\% SR}
        & \textbf{Steps to 80\% SR}
        & \textbf{Steps to 90\% SR} \\
        \midrule
        OpenVLA-OFT + PPO & 60 & 80 & 120 & 180 & - \\
        OpenVLA-OFT + FAN-PPO            & \textbf{40} & \textbf{50} & \textbf{60} & \textbf{100} & \textbf{240}  \\
        \bottomrule
    \end{tabular}
\end{table}

\begin{figure}[htbp]
    \centering
    \includegraphics[width=0.9\textwidth]{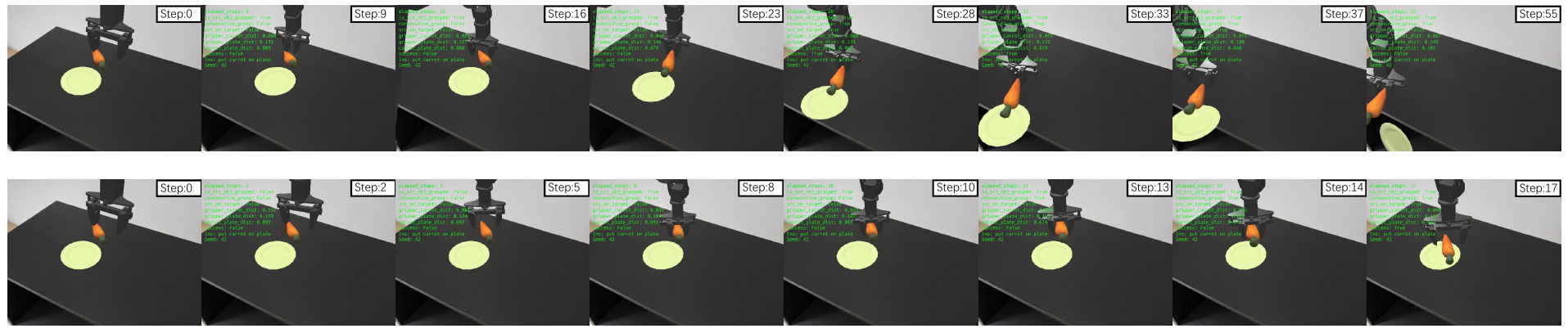}
    \caption{Qualitative comparison of VLA task executions on the in-distribution task.
    Top row: trajectories generated by the baseline OpenVLA + PPO, which fails to accomplish the task.
    Bottom row: trajectories generated by our OpenVLA + FAN-PPO, which successfully complete the task}
    \label{fig_sup:ind}
\end{figure}

\begin{figure}[htbp]
    \centering
    \includegraphics[width=0.9\textwidth]{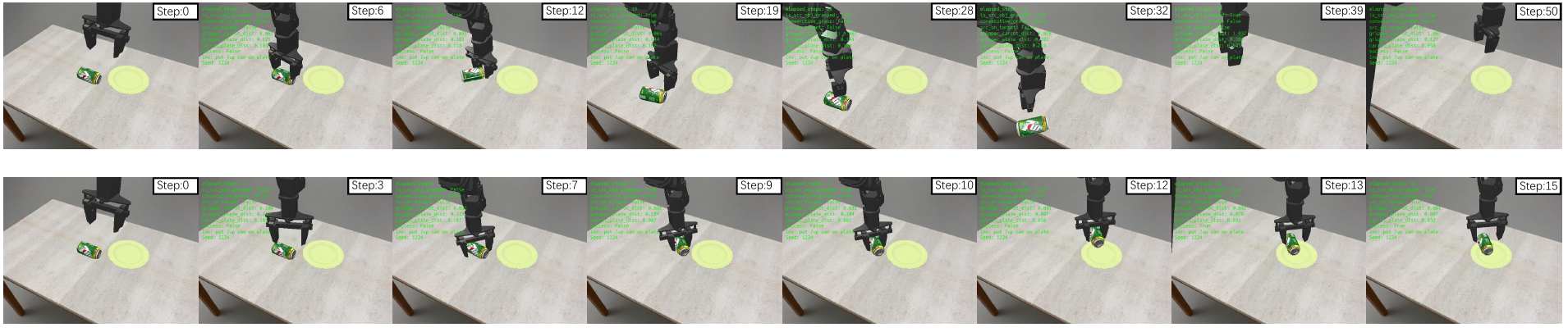}
    \caption{Qualitative comparison of VLA task executions on the 'Table' task under vision OOD conditions.
    Top row: trajectories generated by the baseline OpenVLA + PPO, which fails to accomplish the task.
    Bottom row: trajectories generated by our OpenVLA + FAN-PPO, which successfully complete the task}
    \label{fig_sup:Table}
\end{figure}

\begin{figure}[htbp]
    \centering
    \includegraphics[width=0.9\textwidth]{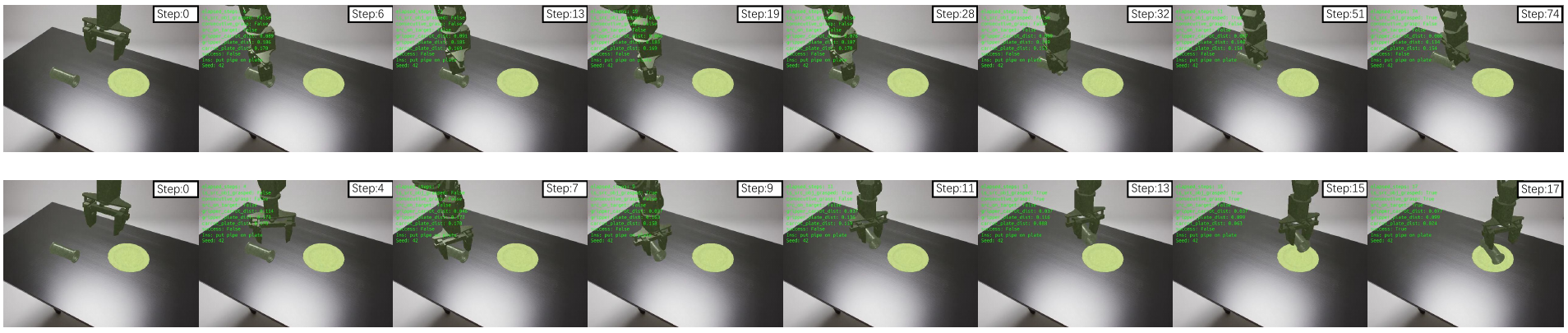}
    \caption{Qualitative comparison of VLA task executions on the 'Texture-w' task under vision OOD conditions.
    Top row: trajectories generated by the baseline OpenVLA + PPO, which fails to accomplish the task.
    Bottom row: trajectories generated by our OpenVLA + FAN-PPO, which successfully complete the task}
    \label{fig_sup:Texture-w}
\end{figure}

\begin{figure}[htbp]
    \centering
    \includegraphics[width=0.9\textwidth]{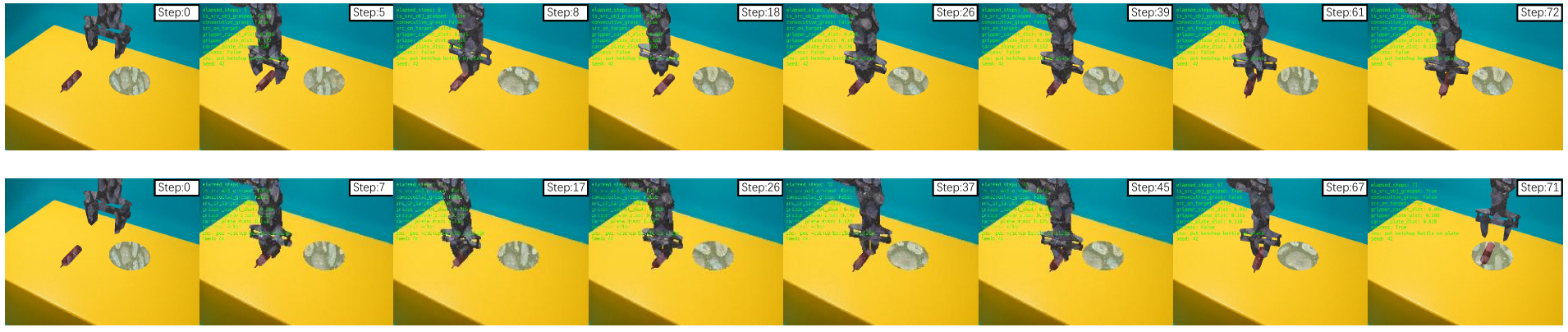}
    \caption{Qualitative comparison of VLA task executions on the 'Texture-s' task under vision OOD conditions.
    Top row: trajectories generated by the baseline OpenVLA + PPO, which fails to accomplish the task.
    Bottom row: trajectories generated by our OpenVLA + FAN-PPO, which successfully complete the task}
    \label{fig_sup:Texture-s}
\end{figure}

\begin{figure}[htbp]
    \centering
    \includegraphics[width=0.9\textwidth]{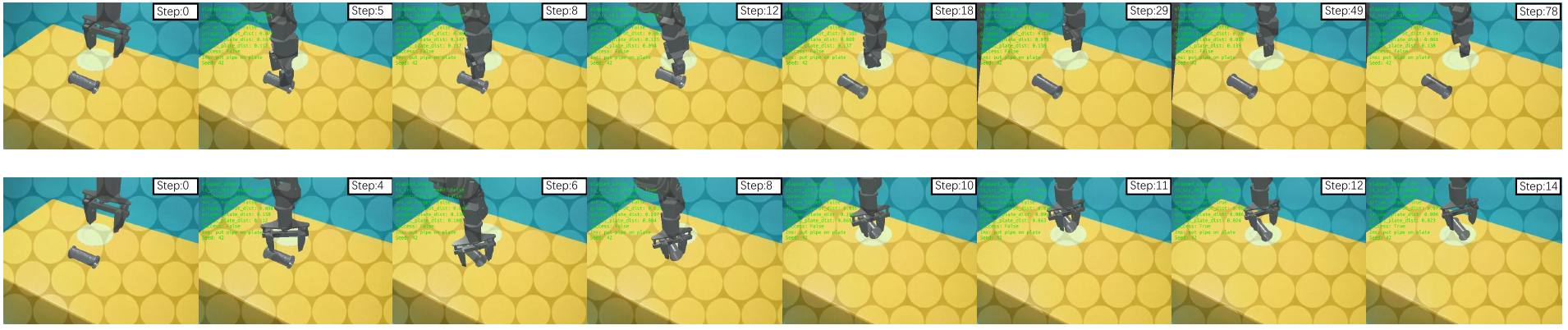}
    \caption{Qualitative comparison of VLA task executions on the 'Noise-w' task under vision OOD conditions.
    Top row: trajectories generated by the baseline OpenVLA + PPO, which fails to accomplish the task.
    Bottom row: trajectories generated by our OpenVLA + FAN-PPO, which successfully complete the task}
    \label{fig_sup:Noise-w}
\end{figure}

\begin{figure}[htbp]
    \centering
    \includegraphics[width=0.9\textwidth]{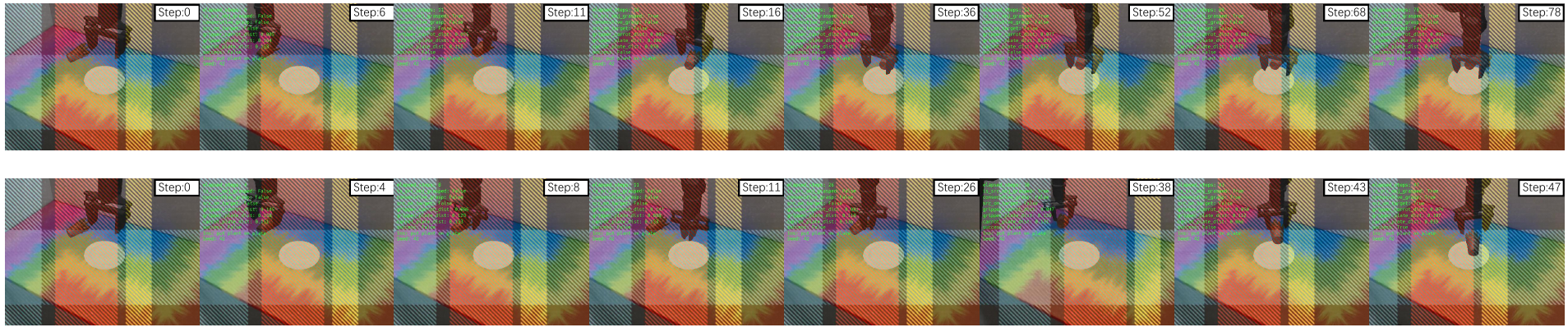}
    \caption{Qualitative comparison of VLA task executions on the 'Noise-s' task under vision OOD conditions.
    Top row: trajectories generated by the baseline OpenVLA + PPO, which fails to accomplish the task.
    Bottom row: trajectories generated by our OpenVLA + FAN-PPO, which successfully complete the task}
    \label{fig_sup:Noise-s}
\end{figure}

\begin{figure}[htbp]
    \centering
    \includegraphics[width=0.9\textwidth]{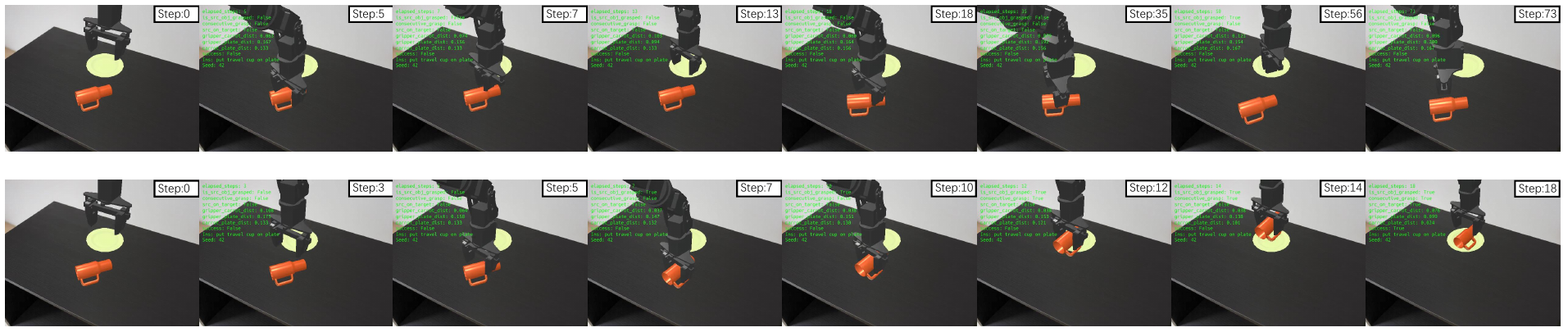}
    \caption{Qualitative comparison of VLA task executions on the 'obj' task under semantic OOD conditions.
    Top row: trajectories generated by the baseline OpenVLA + PPO, which fails to accomplish the task.
    Bottom row: trajectories generated by our OpenVLA + FAN-PPO, which successfully complete the task}
    \label{fig_sup:obj}
\end{figure}

\begin{figure}[htbp]
    \centering
    \includegraphics[width=0.9\textwidth]{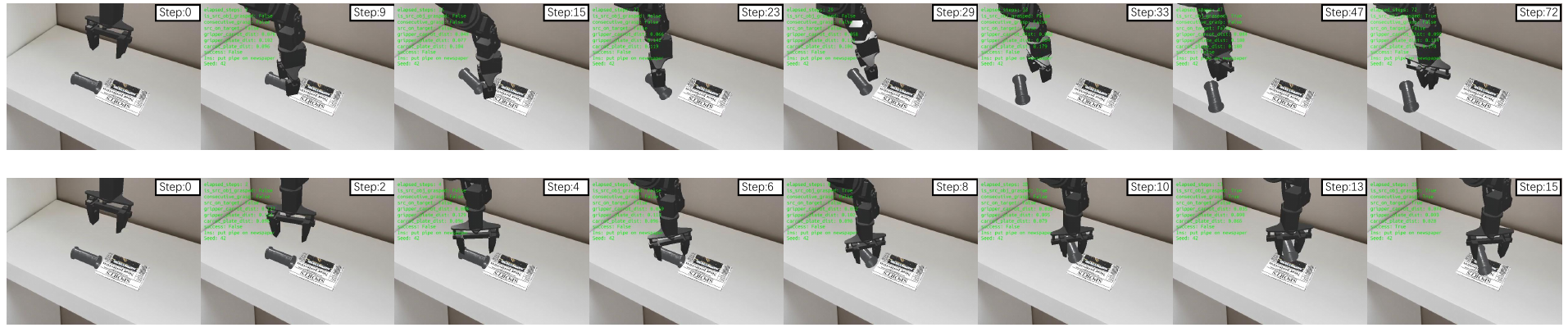}
    \caption{Qualitative comparison of VLA task executions on the 'Recep' task under semantic OOD conditions.
    Top row: trajectories generated by the baseline OpenVLA + PPO, which fails to accomplish the task.
    Bottom row: trajectories generated by our OpenVLA + FAN-PPO, which successfully complete the task}
    \label{fig_sup:Recep}
\end{figure}

\begin{figure}[htbp]
    \centering
    \includegraphics[width=0.9\textwidth]{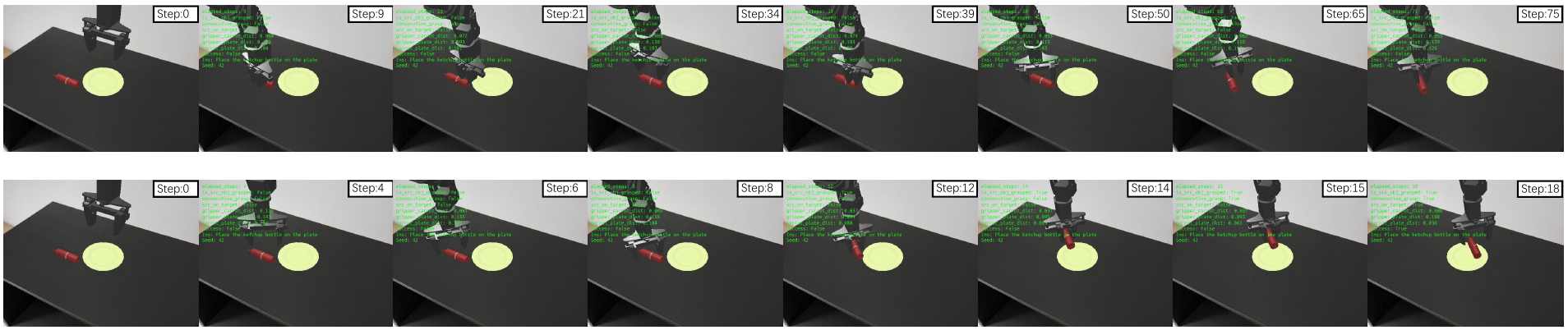}
    \caption{Qualitative comparison of VLA task executions on the 'Instruct' task under semantic OOD conditions.
    Top row: trajectories generated by the baseline OpenVLA + PPO, which fails to accomplish the task.
    Bottom row: trajectories generated by our OpenVLA + FAN-PPO, which successfully complete the task}
    \label{fig_sup:Instruct}
\end{figure}

\begin{figure}[htbp]
    \centering
    \includegraphics[width=0.9\textwidth]{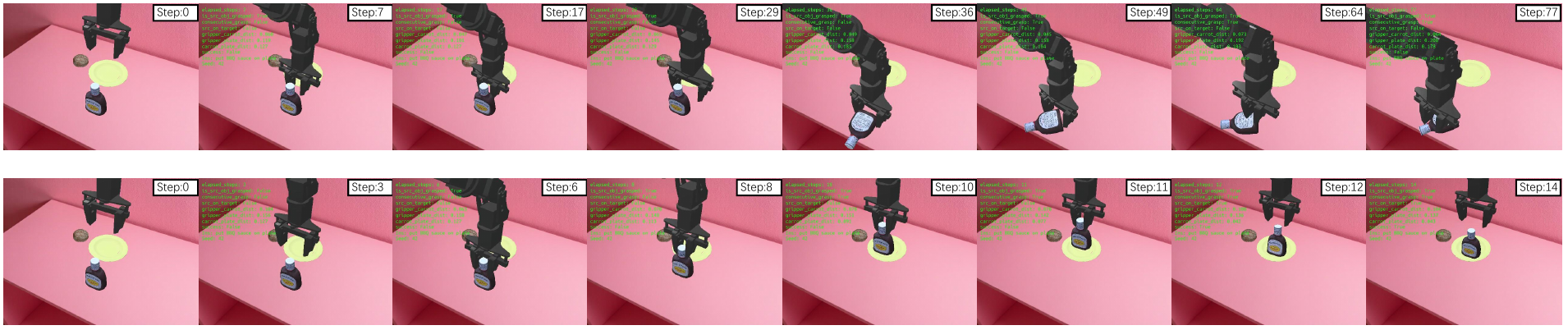}
    \caption{Qualitative comparison of VLA task executions on the 'M-obj (IND)' task under semantic OOD conditions.
    Top row: trajectories generated by the baseline OpenVLA + PPO, which fails to accomplish the task.
    Bottom row: trajectories generated by our OpenVLA + FAN-PPO, which successfully complete the task}
    \label{fig_sup:M-obj(IND)}
\end{figure}

\begin{figure}[htbp]
    \centering
    \includegraphics[width=0.9\textwidth]{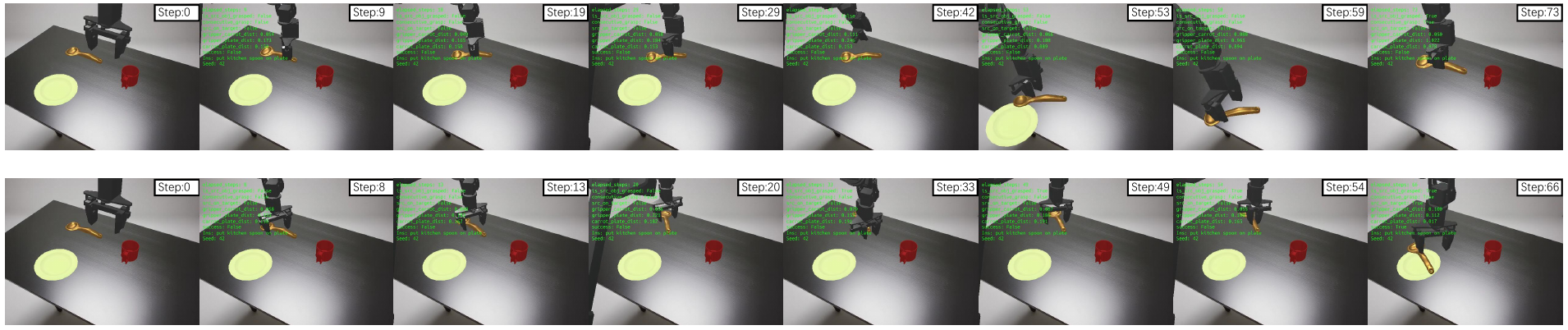}
    \caption{Qualitative comparison of VLA task executions on the 'M-obj (OOD)' task under semantic OOD conditions.
    Top row: trajectories generated by the baseline OpenVLA + PPO, which fails to accomplish the task.
    Bottom row: trajectories generated by our OpenVLA + FAN-PPO, which successfully complete the task}
    \label{fig_sup:M-obj(OOD)}
\end{figure}

\begin{figure}[htbp]
    \centering
    \includegraphics[width=0.9\textwidth]{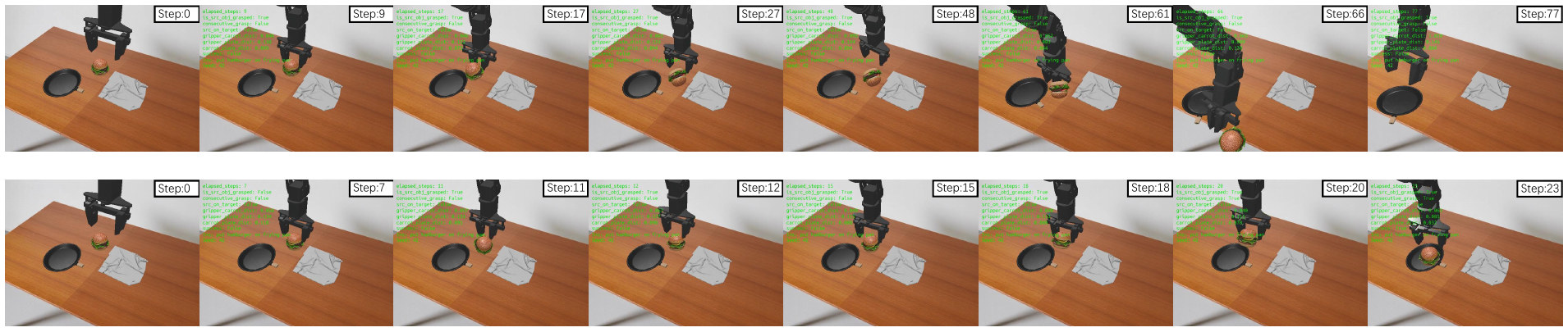}
    \caption{Qualitative comparison of VLA task executions on the 'M-obj (OOD)' task under semantic OOD conditions.
    Top row: trajectories generated by the baseline OpenVLA + PPO, which fails to accomplish the task.
    Bottom row: trajectories generated by our OpenVLA + FAN-PPO, which successfully complete the task}
    \label{fig_sup:M-obj(OOD)}
\end{figure}

\begin{figure}[htbp]
    \centering
    \includegraphics[width=0.9\textwidth]{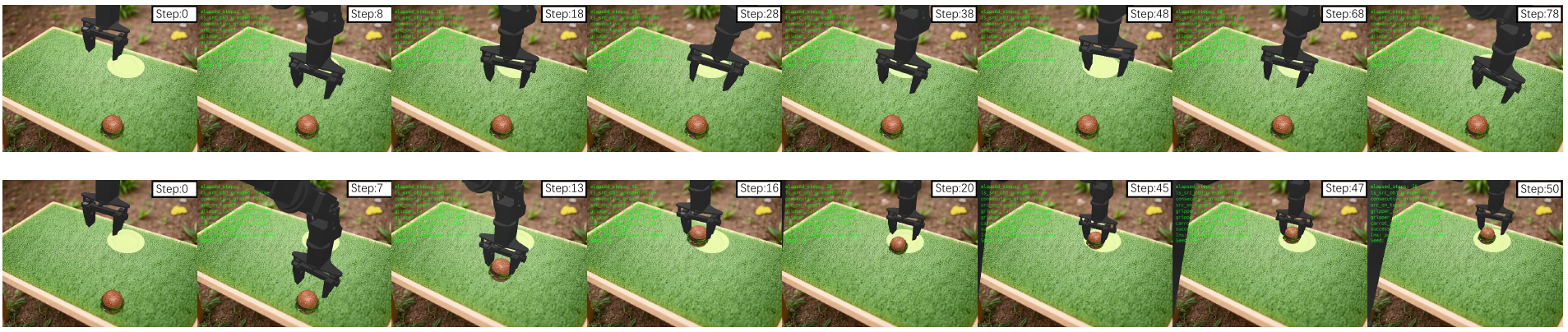}
    \caption{Qualitative comparison of VLA task executions on the 'Obj. Pos.' task under execution OOD conditions.
    Top row: trajectories generated by the baseline OpenVLA + PPO, which fails to accomplish the task.
    Bottom row: trajectories generated by our OpenVLA + FAN-PPO, which successfully complete the task}
    \label{fig_sup:ObjPos}
\end{figure}

\begin{figure}[htbp]
    \centering
    \includegraphics[width=0.9\textwidth]{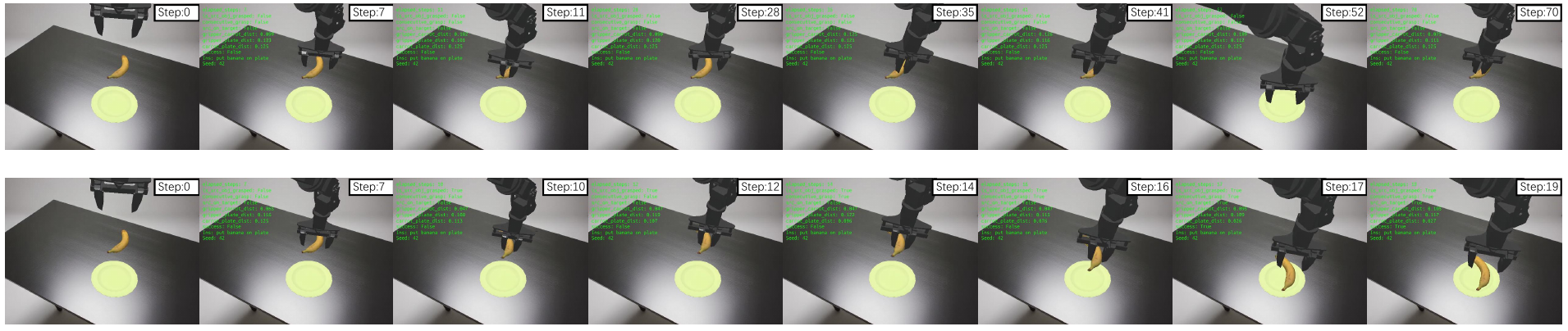}
    \caption{Qualitative comparison of VLA task executions on the 'Robot Pose' task under execution OOD conditions.
    Top row: trajectories generated by the baseline OpenVLA + PPO, which fails to accomplish the task.
    Bottom row: trajectories generated by our OpenVLA + FAN-PPO, which successfully complete the task}
    \label{fig_sup:RobotPose}
\end{figure}

\begin{figure}[htbp]
    \centering
    \includegraphics[width=0.9\textwidth]{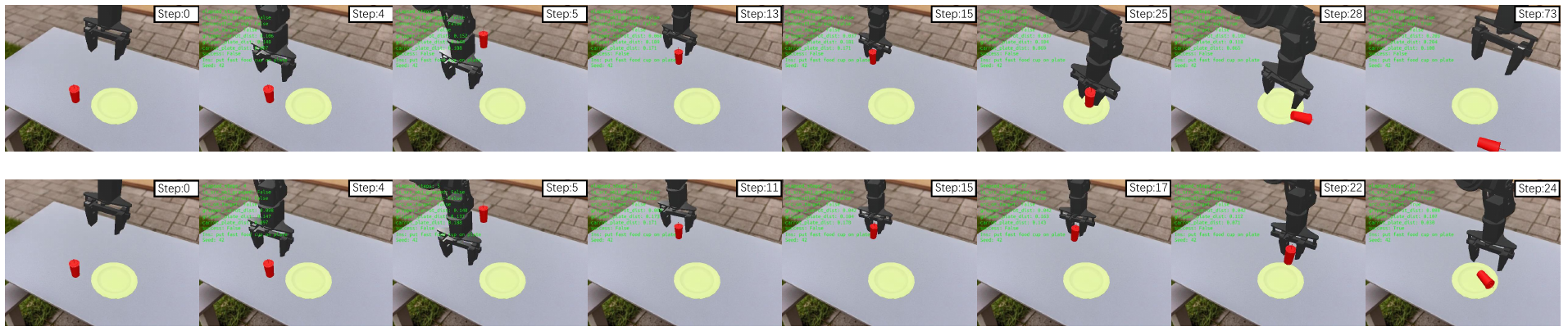}
    \caption{Qualitative comparison of VLA task executions on the 'Obj. Rep.' task under execution OOD conditions.
    Top row: trajectories generated by the baseline OpenVLA + PPO, which fails to accomplish the task.
    Bottom row: trajectories generated by our OpenVLA + FAN-PPO, which successfully complete the task}
    \label{fig_sup:ObjRep}
\end{figure}

\clearpage
\subsection{Parameter Sensitivity Analysis}\label{supp_subsec:parameter}

Figure~\ref{subfig_sup:coef_ind} shows the training dynamics of OpenVLA fine-tuned with FAN-PPO on ManiSkill under different values of the regularization coefficient \(\alpha\).
We observe that excessively large \(\alpha\) can destabilize training (e.g., \(\alpha = 2.0\)) and even cause collapse (e.g., \(\alpha = 5.0\) and \(\alpha = 10.0\)).
Based on these observations, we set \(\alpha = 1.0\) as the default regularization coefficient in all our experiments.

\begin{figure}[htbp]
    \centering
    \subfloat[Rollout success rate on the training subset under different values of \(\alpha\).\label{subfig_sup:coef_ind}]{
        \includegraphics[width=0.41\linewidth]{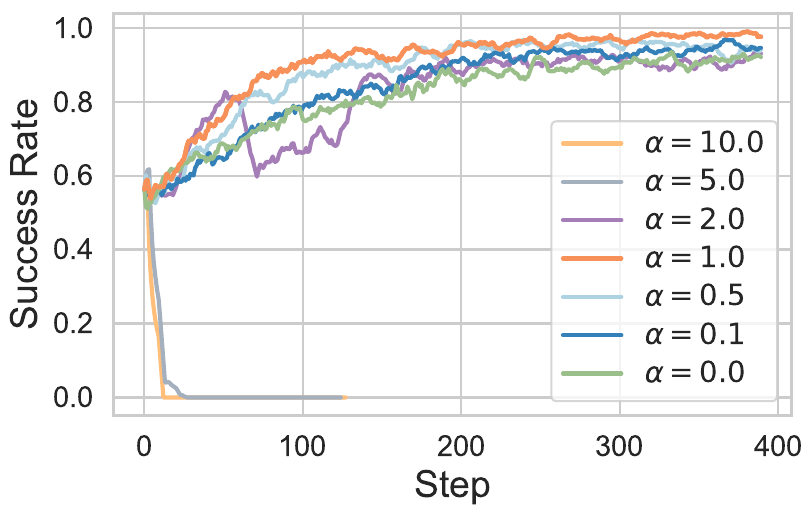}}
    \hspace{10mm}
    \subfloat[Evaluation success rate on the test subset under different values of \(\alpha\).\label{subfig_sup:coef_ood}]{
        \includegraphics[width=0.41\linewidth]{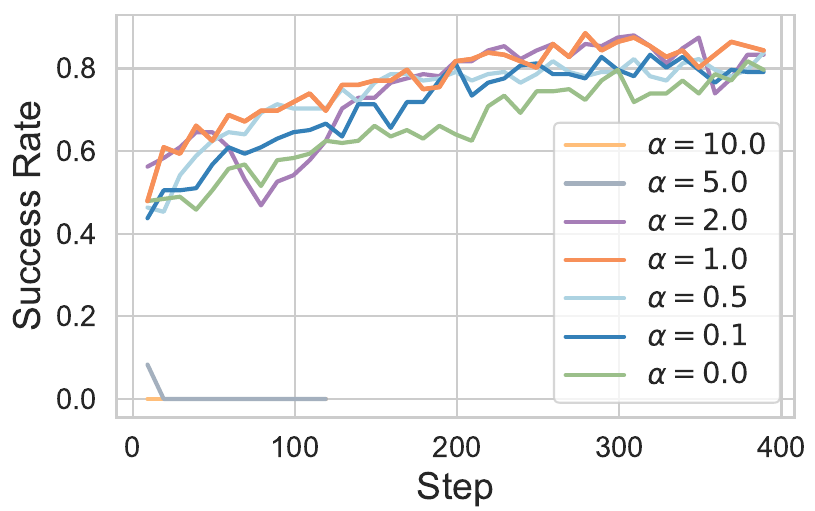}}

    \caption{
    Sensitivity of \textit{OpenVLA + FAN-PPO} to the regularization coefficient \(\alpha\) on the ManiSkill \texttt{PutOnPlateInScene25Main-v3} environment.
    Subfigure~\subref{subfig_sup:coef_ind} reports the rollout success rate on the training subset, while
    subfigure~\subref{subfig_sup:coef_ood} shows the evaluation success rate on the test subset.
    }
    \label{fig_sup:coef}
\end{figure}

Figure~\ref{fig_sup:sigma} shows the training dynamics of OpenVLA fine-tuned with FAN-PPO on ManiSkill under different values of the standard deviation \(\sigma\).
Under our FAN-guided regularization, a smaller \(\sigma\) corresponds to a sharper target distribution. 
When \(\sigma\) becomes too small, the distribution effectively degenerates to a single bin, which leads to degraded performance (e.g., \(\sigma = 0.05\)) or even training collapse (e.g., \(\sigma = 0.01\)).
Once \(\sigma\) exceeds a certain threshold, the model performance becomes relatively insensitive to its exact value: as shown in the figure, values in the range \(\sigma \in [0.1, 2.0]\) yield similar performance on both the in-distribution and OOD subsets.
Based on these observations, we set \(\sigma = 0.3\) as the default standard deviation in all our experiments.

\begin{figure}[htbp]
    \centering
    \subfloat[Rollout success rate on the training subset under different values of \(\sigma\).\label{subfig_sup:sigma_ind}]{
        \includegraphics[width=0.41\linewidth]{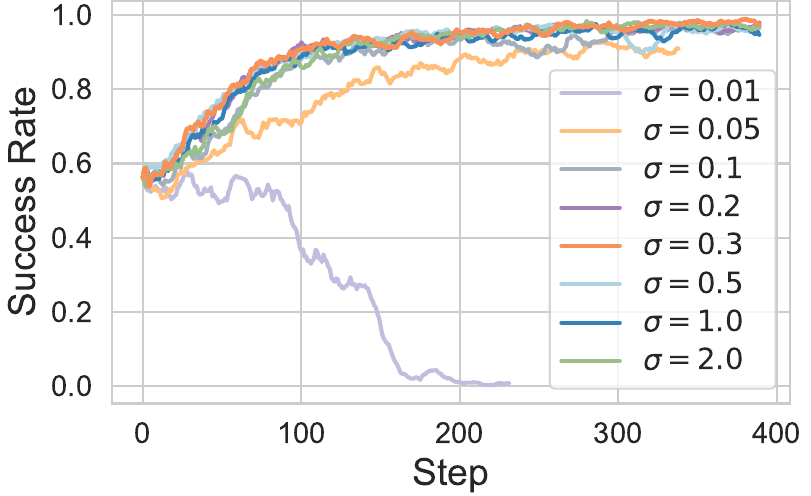}}
    \hspace{10mm}
    \subfloat[Evaluation success rate on the test subset under different values of \(\sigma\).\label{subfig_sup:sigma_ood}]{
        \includegraphics[width=0.41\linewidth]{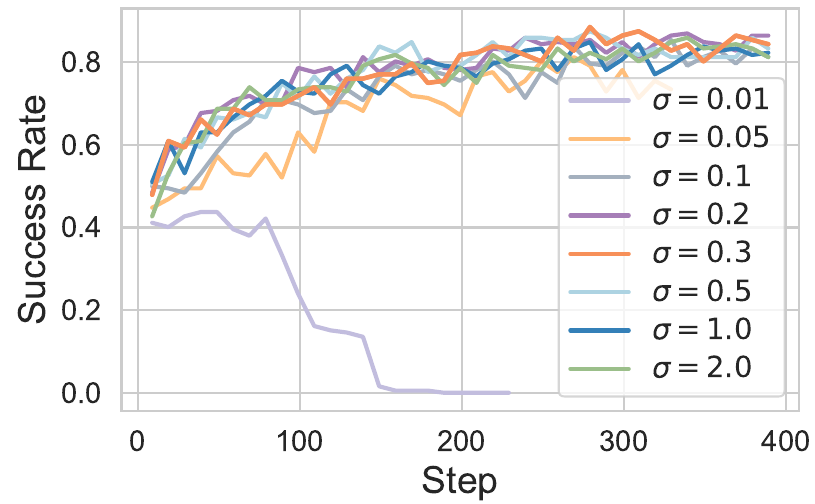}}

    \caption{
    Sensitivity of \textit{OpenVLA + FAN-PPO} to the standard deviation $\sigma$ on the ManiSkill \texttt{PutOnPlateInScene25Main-v3} environment.
    Subfigure~\subref{subfig_sup:sigma_ind} reports the rollout success rate on the training subset, while
    subfigure~\subref{subfig_sup:sigma_ood} shows the evaluation success rate on the test subset.
    }
    \label{fig_sup:sigma}
\end{figure}

We further compare our proposed FAN-PPO method with entropy maximization during the RFT stage, as illustrated in Figure~\ref{fig:entropy}. 
The training curves show that entropy maximization leads to slower performance improvement and lower sample efficiency than FAN-PPO. 
Moreover, its behavior varies more significantly under different choices of the regularization coefficient, indicating weaker training stability and higher sensitivity to hyperparameter selection. 

Our motivation is that feasible actions in manipulation generally exhibit a locally smooth tolerance geometry, for which the Gaussian prior adopted in FAN serves as a simple and straightforward proxy rather than being restricted to this specific form.
To examine this idea, we further test an explicitly multi-modal target by using a Gaussian-kernel-smoothed distribution
$q_{\kappa}=\operatorname{Normalize}(K_{\kappa}\pi)$
in the KL regularization
$D_{\mathrm{KL}}(\pi\|q_{\kappa})$.
The empirical results show that this alternative regularization also yields consistent performance improvements, which supports the general idea of imposing structured local tolerance in action space.
However, as shown in Figure~\ref{fig:kernel}, this multi-modal target regularization still underperforms FAN with a Gaussian prior.
This suggests that the specific design of the regularizer can substantially affect the final performance, and exploring more effective regularization forms remains a worthwhile direction for future work.

\begin{figure}[h]
    \centering
    \subfloat[EM regularization.\label{fig:entropy}]{
        \includegraphics[width=0.41\linewidth]{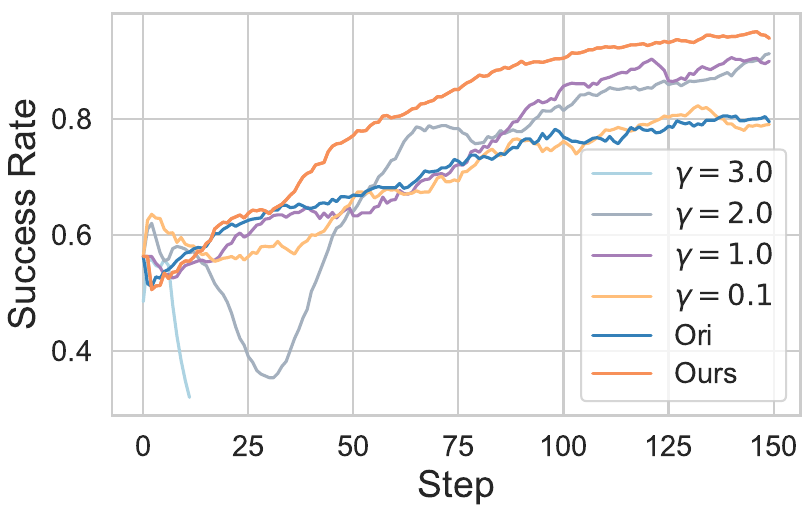}}
    \hspace{10mm}
    \subfloat[Gaussian-kernel smoothing.\label{fig:kernel}]{
        \includegraphics[width=0.41\linewidth]{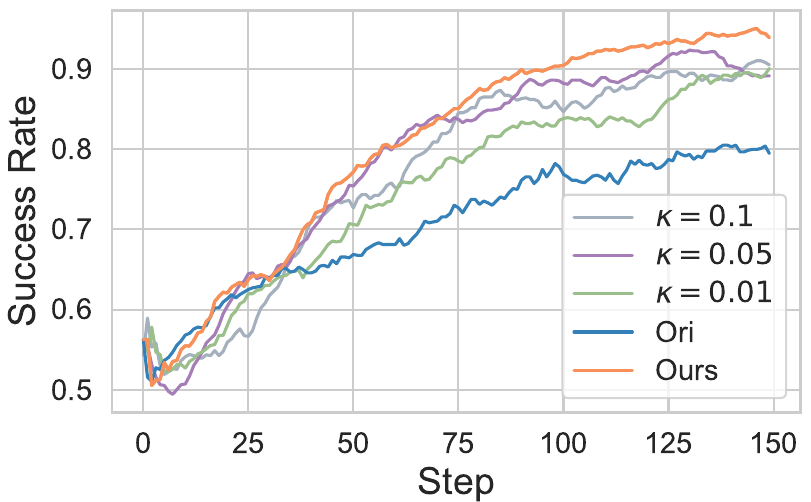}}

    \caption{Training curves during RFT of OpenVLA.}
\end{figure}

\clearpage
\section{Real-World Experiments}

To validate the effectiveness of our method in physical environments, we conducted real-world experiments using a JAKA minicobo manipulator equipped with a fixed Intel RealSense RGB-D camera for visual perception, as illustrated in Figure~\ref{fig_sup:real_world_hardware}. 
Following standard data pre-processing protocols~\cite{kim2024OpenVLA}, we filtered out “no-op” actions from the collected demonstrations—specifically, transitions where the joint-angle changes are near zero and the gripper state remains unchanged—to ensure the high quality of the training data.

We evaluated the policies on a ``pick-and-place'' task involving placing a target object into a designated container. 
To rigorously assess robustness, we designed four distinct testing scenarios ranging from in-distribution settings to challenging configurations with spatial perturbations. 
Qualitative comparisons of the rollout trajectories are presented in Figure~\ref{fig_sup:real_tasks_all}.

In the In-Distribution (IND) setting (Task-1), where the environment setup aligns with the collected demonstrations, both the baseline OpenVLA + SFT and our OpenVLA + FAN-SFT policies successfully complete the task, demonstrating basic competency. 
However, performance diverges significantly under environmental perturbations. 
In Tasks 2--4, we introduced variations in the initial object pose, the manipulator's starting configuration, and the target box position, respectively. 
As shown in the visualizations, the baseline policy lacks the spatial generalization required to adapt to these shifts, frequently failing to grasp the object or missing the target container. 
In contrast, our method demonstrates superior robustness, successfully correcting its trajectory to accommodate these real-world disturbances and consistently completing the manipulation tasks.

\begin{figure}[htbp]
    \centering

    \begin{subfigure}[b]{0.3\textwidth}
        \centering
        \includegraphics[width=\linewidth]{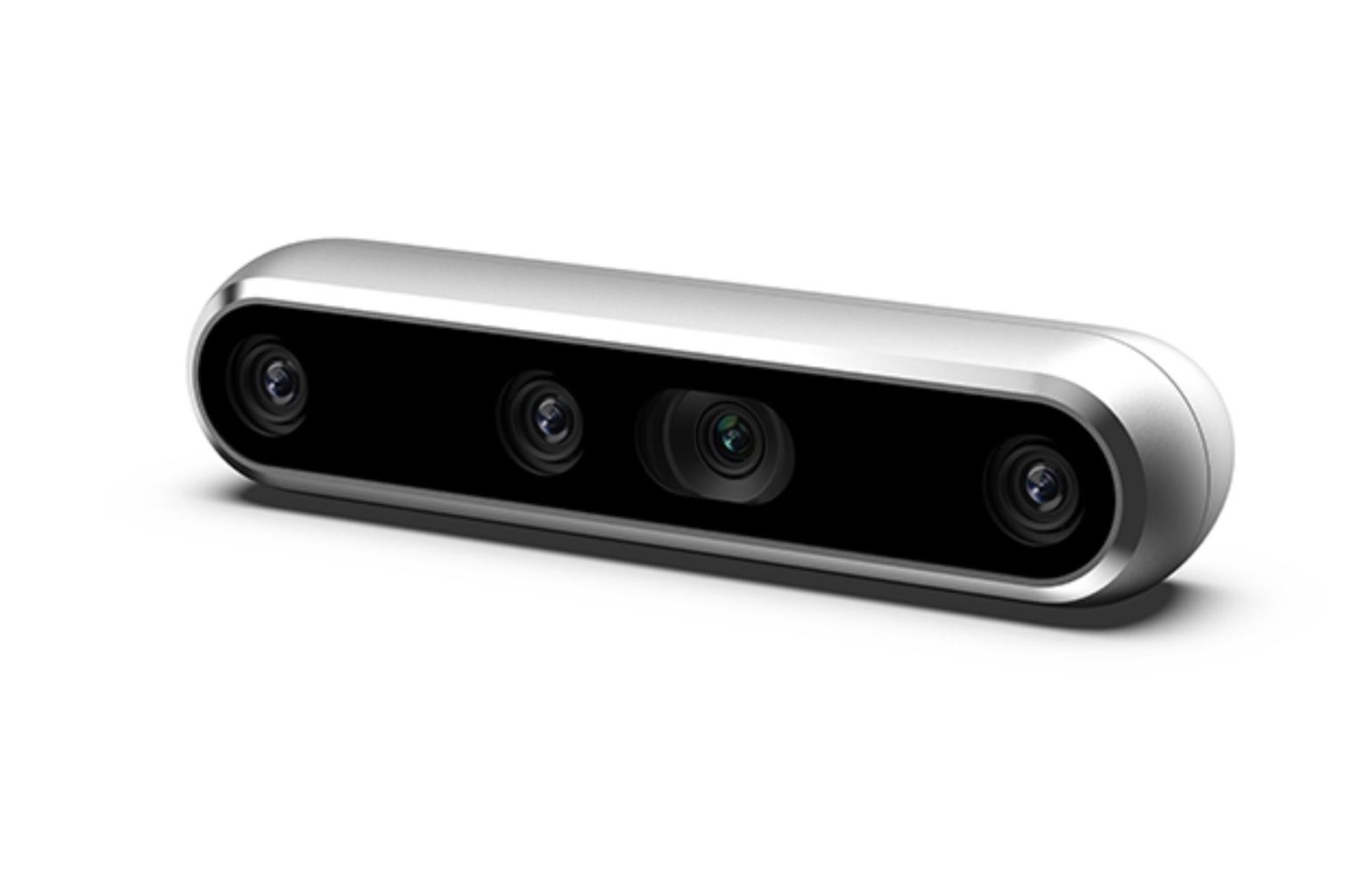}
        \caption{Intel RealSense}
        \label{subfig_sup:camera}
    \end{subfigure}
    \hspace{0.02\textwidth}
    \begin{subfigure}[b]{0.3\textwidth}
        \centering
        \includegraphics[width=\linewidth]{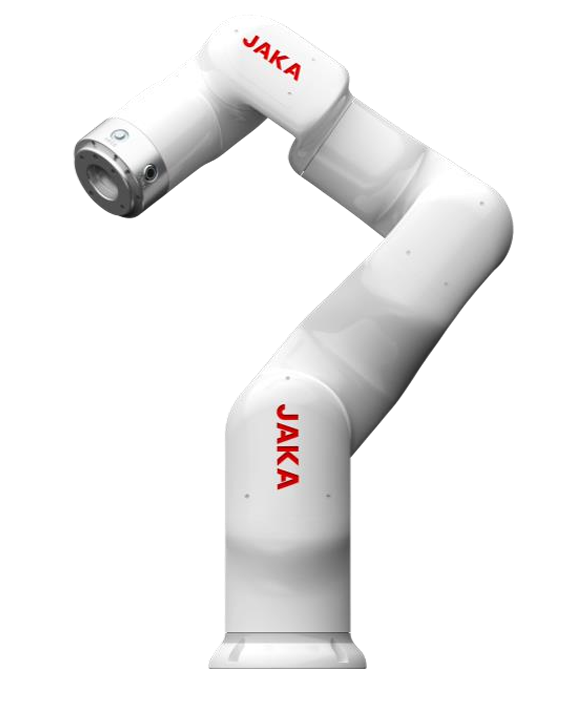}
        \caption{JAKA minicobo Robot}
        \label{subfig_sup:arm}
    \end{subfigure}
    \hspace{0.02\textwidth}
    \begin{subfigure}[b]{0.3\textwidth}
        \centering
        \includegraphics[width=\linewidth, angle=-90]{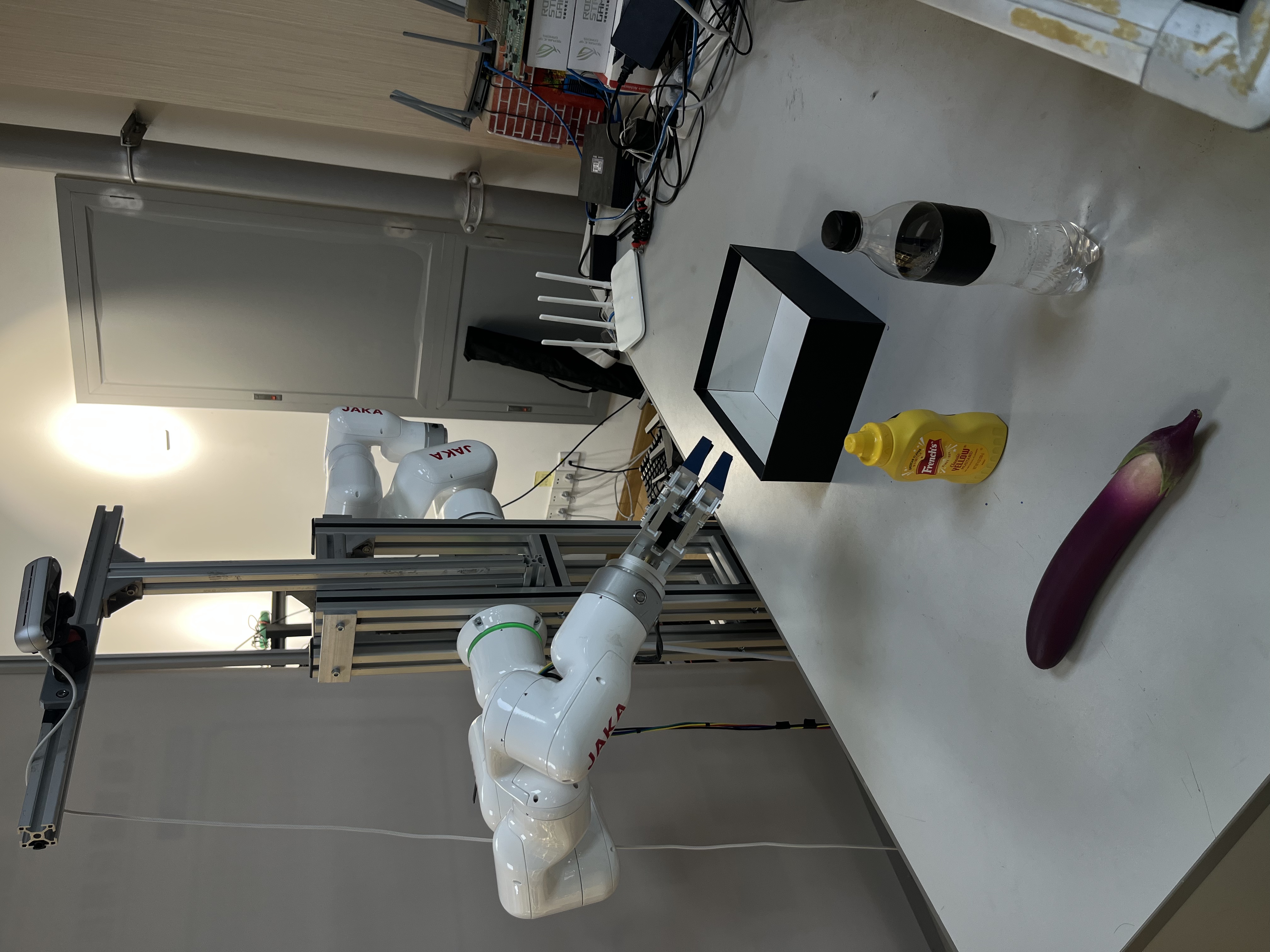}
        \caption{Setup}
        \label{subfig_sup:setup}
    \end{subfigure}

    \caption{Real-world experiment hardware setup. 
    (a) The RGB-D camera used for perception. 
    (b) The robot arm used for manipulation. 
    (c) The overall workspace configuration, where three types of target objects are placed on the table and the task is to place them into the black box.}
    \label{fig_sup:real_world_hardware}
\end{figure}

\begin{figure}[htbp]
    \centering
    \begin{subfigure}{0.9\textwidth}
        \centering
        \includegraphics[width=\textwidth]{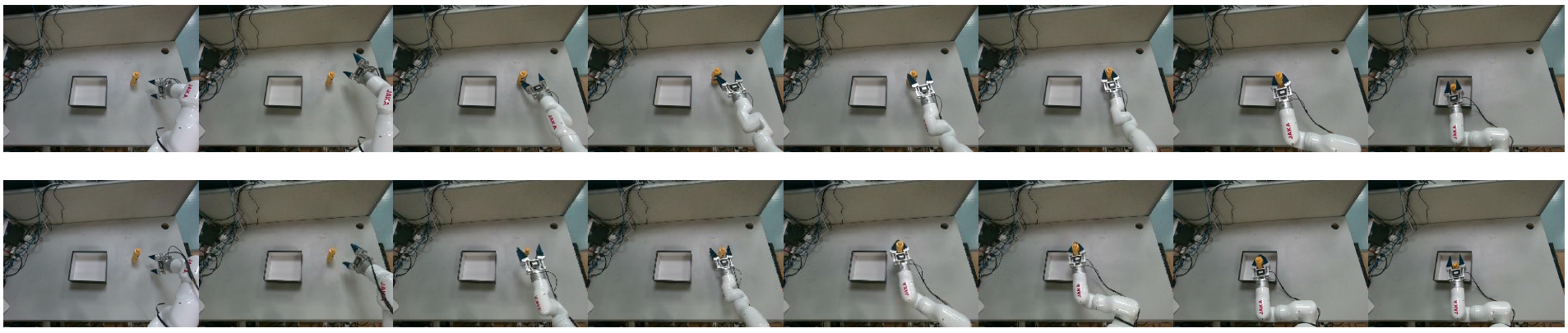}
        \caption{
        \textbf{Task-1} (in-distribution): matched conditions with the collected demonstrations. 
        Both OpenVLA + SFT and OpenVLA + FAN-SFT successfully place the target object into the box.
        }
        \label{fig_sup:real_task1}
    \end{subfigure}

    \vspace{0.6em}

    \begin{subfigure}{0.9\textwidth}
        \centering
        \includegraphics[width=\textwidth]{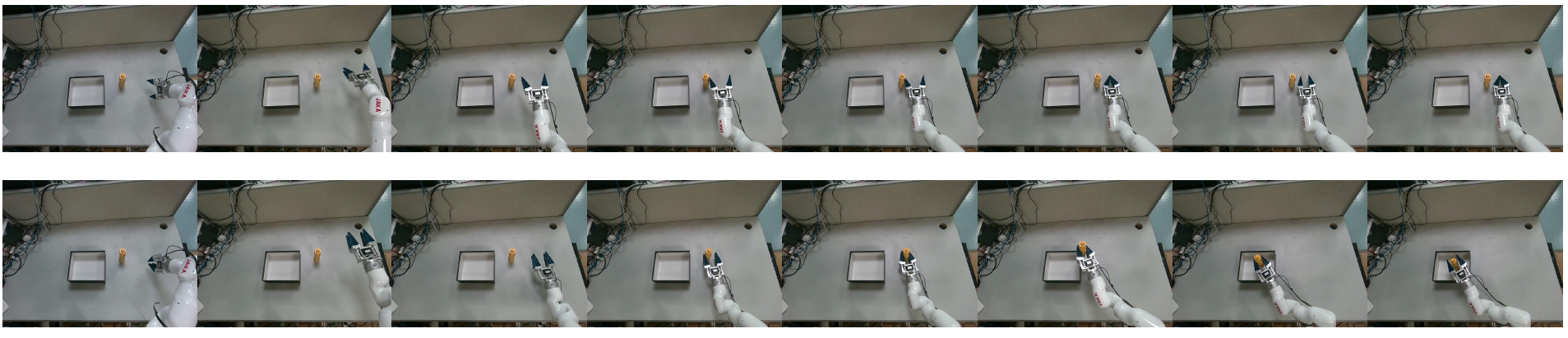}
        \caption{
        \textbf{Task-2}: perturbed initial object pose. 
        OpenVLA + SFT fails to pick up the target object, while OpenVLA + FAN-SFT succeeds under the same perturbation.
        }
        \label{fig_sup:real_task2}
    \end{subfigure}

    \vspace{0.6em}

    \begin{subfigure}{0.9\textwidth}
        \centering
        \includegraphics[width=\textwidth]{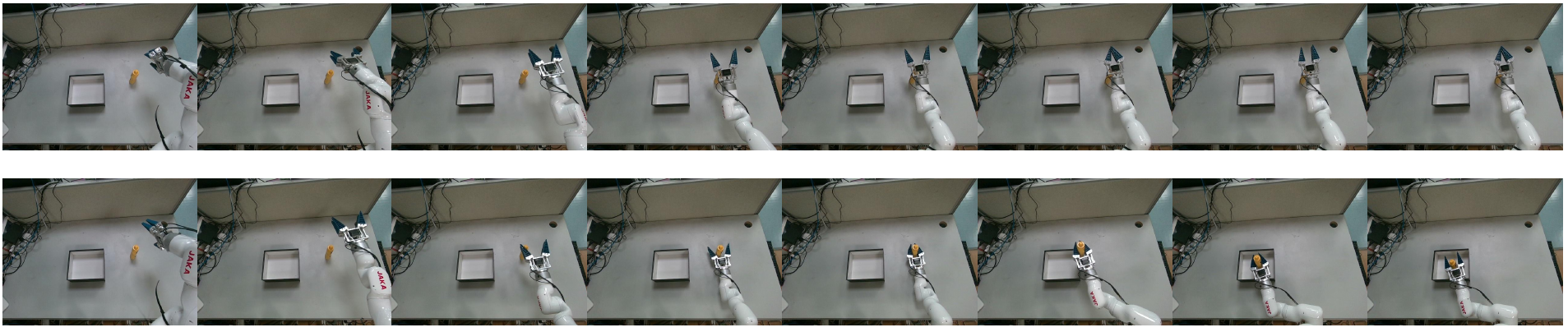}
        \caption{
        \textbf{Task-3}: perturbed initial manipulator pose. 
        OpenVLA + SFT fails to pick up the target object, whereas OpenVLA + FAN-SFT succeeds under the same perturbation.
        }
        \label{fig_sup:real_task3}
    \end{subfigure}

    \vspace{0.6em}

    \begin{subfigure}{0.9\textwidth}
        \centering
        \includegraphics[width=\textwidth]{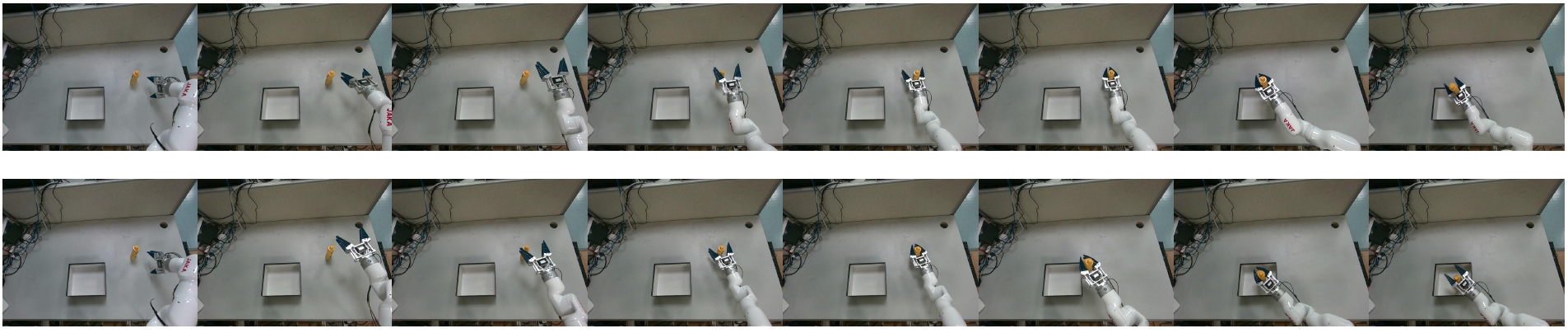}
        \caption{
        \textbf{Task-4}: perturbed initial box position. 
        OpenVLA + SFT fails to place the target object into the box, while OpenVLA + FAN-SFT succeeds under the same perturbation.
        }
        \label{fig_sup:real_task4}
    \end{subfigure}

    \caption{
    Real-world rollouts on four object-in-box manipulation tasks.
    In each subfigure, the top row shows time-ordered key frames of trajectories generated by the baseline OpenVLA + SFT policy, 
    and the bottom row shows trajectories generated by our proposed OpenVLA + FAN-SFT policy.
    Task-1 corresponds to the in-distribution setting with matched demonstration conditions, where both methods succeed.
    Tasks~2--4 introduce different real-world perturbations (object pose, manipulator pose, and box position, respectively), 
    under which the baseline policy fails while our method remains successful.
    }
    \label{fig_sup:real_tasks_all}
\end{figure}

\end{document}